\documentclass[twoside,11pt]{article}
\usepackage[bookmarks,colorlinks,breaklinks]{hyperref}  
\hypersetup{linkcolor=blue,citecolor=blue,filecolor=blue,urlcolor=blue} 
\usepackage{fullpage}
\usepackage{amsmath,amsfonts,amsthm,amssymb,xspace,bm}
\usepackage{graphicx}
\usepackage{url,algorithm,algorithmic}

\newcommand{\newref}[2][]{\hyperref[#2]{#1~\ref*{#2}}}
\renewcommand{\eqref}[1]{\hyperref[#1]{(\ref*{#1})}}

\theoremstyle{plain}
\newtheorem{theorem}{Theorem}[section]
\newtheorem{lemma}[theorem]{Lemma}

\newtheorem{corollary}[theorem]{Corollary}

\newtheorem{definition}[theorem]{Definition}

\theoremstyle{definition}

\usepackage{amsmath, bm, mathrsfs}
\usepackage{graphicx}
\usepackage{epsf}
\usepackage{algorithm, algorithmic}
\newcommand{\A}{W}
\renewcommand{\AA}{\widetilde{W}}

\newcommand{\burg}{D_{\ell\text{d}}}
\newcommand{\von}{D_{\text{vN}}}
\newcommand{\fvN}{f_{\text{vN}}}
\newcommand{\ffrob}{f_{\text{frob}}}
\newcommand{\frob}{D_{\text{frob}}}
\newcommand{\dist}{d_\A}

\newcommand{\argmin}{\operatornamewithlimits{argmin}}

\newcommand{\X}{X}
\newcommand{\XX}{\widetilde{X}}
\newcommand{\C}{C}
\renewcommand{\S}{{\mathcal S}}
\newcommand{\D}{{\mathcal D}}
\renewcommand{\b}{b}
\newcommand{\K}{K}
\newcommand{\Z}{Z}
\newcommand{\x}{\bm{x}}
\newcommand{\xx}{\tilde{\bm{x}}}
\newcommand{\y}{\bm{y}}
\renewcommand{\v}{\bm{v}}
\newcommand{\q}{\bm{q}}
\newcommand{\e}{\bm{e}}

\DeclareMathOperator{\Tr}{tr}
\DeclareMathOperator{\tr}{tr}

\numberwithin{equation}{section}

\begin{document}

\title{Metric and Kernel Learning using a Linear Transformation}
\author{Prateek Jain \and Brian Kulis \and Jason V. Davis \and Inderjit S. Dhillon}
\maketitle

\begin{abstract}

Metric and kernel learning are important in several machine 
learning applications.  However, most existing metric learning algorithms are
limited to learning metrics over low-dimensional data, while existing kernel
learning algorithms are often limited to the transductive setting and do not 
generalize to new data points.  In this paper, we study metric learning as a
problem of learning a linear transformation of the input data. We show that for
high-dimensional data, a particular framework for
learning a linear transformation of the data based on the LogDet divergence
can be efficiently kernelized to learn a metric (or equivalently, a kernel
function) over an arbitrarily high dimensional space.  We further demonstrate
that a wide class of convex loss functions for learning linear
transformations can similarly be kernelized, thereby considerably expanding the potential
applications of metric learning. We demonstrate our learning approach by applying it to large-scale real world problems in computer 
vision and text mining.

\end{abstract}

\section{Introduction}
One of the basic requirements of many machine learning algorithms (e.g., semi-supervised clustering algorithms, nearest neighbor classification algorithms)  is the
ability to compare two objects to compute a similarity or distance between
them.  In many cases, off-the-shelf distance or similarity functions such as the Euclidean
distance or cosine similarity are used; for example, in text retrieval applications, the cosine
similarity is a standard function to compare two text documents.  However,
such standard distance or similarity functions are not appropriate for all problems.

Recently, there has been significant effort focused on learning how to
compare data objects.  One approach has been to learn a distance metric between
objects given additional side information such as pairwise similarity and
dissimilarity constraints over the data.  

One class of distance metrics that has shown excellent
generalization properties is the \textit{Mahalanobis distance}
function~\cite{itml,xing,wein,roweis,shalev}.  The Mahalanobis distance can be
viewed as a method in which data is subject to a
{\em linear transformation}, and then distances in this transformed space are
computed via the
standard squared Euclidean distance.  Despite their simplicity and
generalization ability, Mahalanobis distances suffer
from two major drawbacks: 1) the number of parameters grows quadratically
with the dimensionality of the data, making it difficult to learn distance
functions over high-dimensional data, 2) learning a linear transformation is
inadequate for data sets with non-linear decision boundaries.  

To address the latter shortcoming, {\it kernel learning} algorithms typically
attempt to learn a 
kernel matrix over the data.  
Limitations of linear methods can be overcome by
employing a non-linear input kernel, which effectively maps the data non-linearly to a high-dimensional feature space.  However, many existing kernel learning
methods are still limited in that the learned kernels do not generalize to
new points~\cite{kwok:tsang:03,kulis,tsuda}.  These methods are restricted
to learning in the transductive 
setting where all the data (labelled and unlabeled) is assumed to be given
upfront.  There has been some work on learning kernels that generalize to new
points, most notably work on hyperkernels~\cite{ong03}, but the resulting
optimization problems are expensive and cannot be scaled to large or even
medium-sized data sets. 

In this paper, we explore metric learning with linear transformations over
arbitrarily high-dimensional spaces; as we will see, this is equivalent to
learning a parameterized kernel function 
$\phi(\bm{x})^T W \phi(\bm{y})$ given an input kernel function
$\phi(\bm{x})^T \phi(\bm{y})$.
In the first part of the paper, we focus on a particular loss function called the LogDet divergence, for
learning the positive definite matrix $\A$.  This loss function is advantageous for several
reasons: it is defined only over positive definite matrices, which makes
the optimization simpler, as we will be able to effectively ignore the
positive definiteness constraint on $\A$.  The loss function has precedence in
optimization~\cite{fletcher} and statistics~\cite{stein}.  An important advantage of our method is that the proposed optimization algorithm is scalable to very large data sets of the order of millions of data
objects.  But perhaps most importantly,
the loss function permits efficient kernelization, allowing the learning of a linear transformation in kernel space. As a result, unlike transductive kernel learning methods, our method easily
handles out-of-sample extensions, i.e., it can be applied to unseen data.

Later in the paper, we extend our result on kernelization of the LogDet
formulation to other convex loss functions for learning $W$, 
and give conditions for which we are able to compute and evaluate the learned
kernel functions.  Our result is akin to the representer theorem for
reproducing kernel Hilbert spaces, where the optimal parameters can be
expressed purely in terms of the training data.  In our case, even though the
matrix $W$ may be infinite-dimensional, it can be fully represented in terms of
the constrained data points, making it possible to compute the learned kernel
function value over arbitrary points.  

Finally, we apply our algorithm to a number of challenging learning problems,
including ones from the domains of computer vision and text mining.  Unlike
existing 
techniques, we can learn linear transformation-based distance  
or kernel functions over these domains, and we show that the resulting
functions lead to improvements over state-of-the-art techniques for a variety
of problems.  


\vspace*{-10pt}
\section{Related Work}
Most of the existing work in metric learning has been done in
the Mahalanobis distance (or metric) learning paradigm, which has
been found to be a sufficiently powerful class of metrics for a variety of 
different data.  One of the earliest papers on metric learning \cite{xing}
proposes a semidefinite programming formulation under similarity and
dissimilarity constraints for learning a Mahalanobis distance, but the
resulting formulation is slow to optimize and has been outperformed by more
sophisticated techniques.  More recently, \cite{wein} formulate the metric
learning problem in a large margin setting, with a focus on $k$-NN
classification. They also formulate the problem as a semidefinite programming
problem and consequently solve it using a method that combines sub-gradient descent
and alternating projections. \cite{roweis} proceed to learn a linear
transformation in the fully supervised setting.  Their formulation seeks to
`collapse classes' by constraining within-class distances to be zero while
maximizing between-class distances.  While each of these algorithms was
shown to yield improved classification performance over the baseline metrics,
their constraints do not generalize outside of their particular problem
domains; in contrast, our approach allows arbitrary linear constraints on the
Mahalanobis matrix.  Furthermore, these algorithms all require eigenvalue
decompositions or semi-definite programming, an operation that is cubic in the dimensionality of the data. 

Other notable work where the authors present methods for learning
Mahalanobis metrics includes \cite{shalev} (online metric learning),
Relevant Components Analysis (RCA) \cite{rca} (similar to discriminant
analysis), locally-adaptive discriminative methods \cite{hastie}, and
learning from relative comparisons \cite{joachim}.  In particular, the
method of \cite{shalev} provided the first demonstration of Mahalanobis
distance learning in kernel space.  Their construction, however, is expensive
to compute, requiring cubic time per iteration to update the parameters.  As
we will see, our LogDet-based algorithm can be implemented more efficiently.

Non-linear transformation based metric learning methods have also been proposed,
though these methods usually suffer from suboptimal performance,
non-convexity, or computational complexity. Some example methods
include neighborhood component analysis (NCA) \cite{nca} that learns a
distance metric specifically for nearest-neighbor based
classification; the convolutional neural net based method of
\cite{chopra}; and a general Riemannian metric learning method
\cite{lebanon}.

There have been several recent papers on kernel learning.  As mentioned in
the introduction, much of the research is limited to learning in the
transductive setting, e.g. ~\cite{kwok:tsang:03,kulis,tsuda}.  Research on
kernel learning that does generalize to new data points includes 
multiple kernel learning~\cite{lanck}, where a linear combination of base 
kernel functions are learned; this approach has proven to be useful for a
variety of problems, such as object recognition in
computer vision.  Another approach to kernel learning is to use
hyperkernels~\cite{ong03}, which consider functions between kernels,
and learn in the appropriate reproducing kernel Hilbert space between such
functions.  In both cases, semidefinite programming is used, making the
approach impractical for large-scale learning problems.  Recently, some work
has been done on making hyperkernel learning more efficient via second-order
cone programming~\cite{hyperkernel_socp}, however this formulation still
cannot be applied to large data sets.  Concurrent to our work in showing
kernelization for a wide class of convex loss functions, a recent paper
considers kernelization of other Mahalanobis distance learning algorithms
such as LMNN and NCA~\cite{arxiv}.  The latter paper, which appeared after the
conference version of the results in our paper, presents a representer-type
theorem and can be seen as complementary to the general kernelization results
(see Section~\ref{sec:loss}) we present in this paper. 

The research in this paper extends work done in~\cite{itml},~\cite{kulis},
and~\cite{jasonkdd}.  While the focus in~\cite{itml} and~\cite{jasonkdd} was
solely on the LogDet divergence, in this work we characterize kernelization
of a wider class of convex loss functions.  Furthermore, we provide a more
detailed analysis of kernelization for the Log Determinant loss, and include
experimental results on large scale kernel learning.  We extend the work
in~\cite{kulis} to the inductive setting; the main goal in \cite{kulis} was
to demonstrate the computational benefits of using the LogDet and
von Neumann divergences for learning low-rank kernel matrices. Finally in this paper, we do not consider online models for metric and kernel learning, however interested readers can refer to \cite{onlinemetric_nips}. 


\newcommand{\dlr}{IPLR}
\section{Metric and Kernel Learning via the LogDet Divergence}\label{sec:alg}
In this section, we introduce the LogDet formulation for linearly transforming the data given a set of pairwise distance constraints.  As
discussed below, this is equivalent to a Mahalanobis metric learning
problem.   
We then discuss kernelization issues of the formulation
and present efficient optimization algorithms.  Finally, we address
limitations of the method when the amount of training data is large, and
propose a modified algorithm to efficiently learn a kernel under such
circumstances.

\subsection{Mahalanobis Distances and Parameterized Kernels}
\label{sec:itml_prob}
First we introduce the framework for metric and kernel learning that is
employed in this paper.  Given a data set of objects $X = [\bm{x}_1,
..., \bm{x}_n], \bm{x}_i\in\mathbb{R}^d$ (when working in kernel space, the data matrix will be
represented as $X = [\phi(\bm{x}_1), ..., \phi(\bm{x}_n)]$, where $\phi$ is
the mapping to feature space), 
we are interested in finding an appropriate distance function to compare two
objects.  We consider the Mahalanobis distance, parameterized by a positive
definite matrix $\A$; the squared distance between two points $\bm{x}_i$ and
$\bm{x}_j$ is given by
\begin{displaymath}
\dist(\bm{x}_i,\bm{x}_j) = (\bm{x}_i - \bm{x}_j)^T \A (\bm{x}_i - \bm{x}_j).
\end{displaymath}
This distance function can be viewed as learning a linear transformation of
the data and measuring the squared Euclidean distance in the transformed
space.  This is seen by factorizing the matrix $\A = G^T G$ and observing that
$\dist(\bm{x}_i,\bm{x}_j) = \|G\bm{x}_i - G \bm{x}_j\|_2^2$.  However, if the data
is not linearly separable in the input space, then the resulting distance
function may not be powerful enough for the desired application.  As a
result, we are interested in working in kernel space; that is, we can express
the Mahalanobis distance in kernel space after applying an appropriate
mapping $\phi$ from input to feature space:
\begin{displaymath}
d_\A(\bm{x}_i,\bm{x}_j) = (\phi(\bm{x}_i) - \phi(\bm{x}_j))^T \A (\phi(\bm{x}_i) - \phi(\bm{x}_j)).
\end{displaymath}
As is standard with kernel-based algorithms, we require that this distance be
computable given the ability to compute the kernel function
$\kappa_0(\bm{x},\bm{y}) = \phi(\bm{x})^T \phi(\bm{y})$.  We can
therefore equivalently pose the problem as learning a parameterized kernel
function $\kappa(\bm{x},\bm{y})=\phi(\bm{x})^T \A \phi(\bm{y})$ given some input kernel function
$\kappa_0(\bm{x},\bm{y})=\phi(\bm{x})^T \phi(\bm{y})$.

To learn the resulting metric/kernel, we assume that we are given
constraints on the desired distance function.  In this paper, we assume that
pairwise similarity and dissimilarity constraints are given over the
data---that is, pairs of points that should be similar under the learned
metric/kernel, and pairs of points that should be dissimilar under the
learned metric/kernel.
Such constraints are natural in many settings; for example, given class labels over
the data, 
points in the same class should be similar to one another and dissimilar to
points in different classes.  However, our approach is general and can
accommodate other potential constraints over the distance function, such as
relative distance constraints.  

The main challenge is in finding an appropriate loss function for
learning the matrix $\A$ so that 1) the resulting algorithm is scalable and efficiently computable in kernel space, 2) the resulting metric/kernel yields improved performance on the underlying machine learning problem, such as classification, semi-supervised clustering etc.  We now move on to the details.
\subsection{LogDet Metric Learning}


The LogDet divergence between two positive definite matrices\footnote{The definition of LogDet divergence can be extended to the case when $\A_0$ and $\A$ are rank deficient by appropriate use of the pseudo-inverse. The interested reader may refer to \cite{kulis}.} $\A$, $\A_0 \in \mathbb{R}^{d\times d}$ is defined to be
$$\burg(\A, \A_0)=\tr(\A\A_0^{-1})-\log \det(\A\A_0^{-1})-d.$$
We are interested in finding $\A$ that is closest to $\A_0$ as measured by the LogDet divergence but that
satisfies our desired constraints. When $\A_0=I$, this formulation can be interpreted as a maximum entropy problem. 
Given a set of similarity constraints $S$ and dissimilarity constraints
$D$, we propose the following problem:
\begin{equation}
   \label{eq:itml}
  \begin{split}
    \min_{\A \succeq 0} \quad & \burg(\A, I)\\
    \text{s.t.}\quad & d_\A(\bm{x}_i,\bm{x}_j) \leq u, \qquad (i,j) \in {\mathcal S},\\ 
    & d_\A(\bm{x}_i,\bm{x}_j) \geq \ell, \qquad\ (i,j) \in {\mathcal D}.
  \end{split}
\end{equation}
The above problem was considered in~\cite{itml}. LogDet has many important properties that make it useful for machine learning
and optimization, including scale-invariance and preservation of the range space.  Please see~\cite{kulis_jmlr} for a detailed discussion on the
properties of LogDet.  Beyond 
this, we prefer LogDet over other loss functions (including the squared
Frobenius loss as used in~\cite{shalev} or a linear objective as
in~\cite{wein}) due to the fact that the resulting algorithm turns out to be simple and
efficiently kernelizable.  We note that formulation~\eqref{eq:itml} minimizes the LogDet
divergence to the identity matrix $I$.  This can be generalized to arbitrary
positive definite matrices $\A_0$, however without loss of generality we can consider $\A_0=I$ since $\burg(\A,\A_0)=\burg(\A_0^{-1/2}\A\A_0^{-1/2},I)$. Further, formulation \eqref{eq:itml} considers simple similarity and dissimilarity constraints
over the learned Mahalanobis distance, but other linear constraints are
possible.  Finally, the above
formulation assumes that there exists a feasible solution to the proposed
optimization problem; extensions to the infeasible case involving slack
variables are discussed later (see Section~\ref{sec:itml_alg}). 

\subsection{Kernelizing the Problem}
\label{sec:itml_kernel}
We now consider the problem of kernelizing the metric learning problem. Subsequently, we will present an efficient algorithm and discuss generalization to new points.

Given a set of $n$ constrained data points, let $\K_0$ denote the input kernel
matrix for the data, i.e. $\K_0(i,j) = \kappa(\bm{x}_i,\bm{x}_j) =
\phi(\bm{x}_i)^T\phi(\bm{x}_j)$.  Note that the squared Mahalanobis distance
in kernel space may be written as $\dist(\phi(\x_i), \phi(\x_j))=\K(\x_i,
\x_i)+\K(\x_j,\x_j)-2\K(\x_i,\x_j)$, where $K$ is the learned kernel matrix;
equivalently, we may write the squared distance as
$\Tr(\K(\bm{e}_i-\bm{e}_j)(\bm{e}_i - \bm{e}_j)^T)$, where $\bm{e}_i$ is the
$i$-th canonical basis vector. 
Consider the following problem to find $\K$:
\begin{align}
\label{eq:itml_kernel}
  \min_{\K \succeq 0} &\quad \burg(\K, \K_0) \nonumber\\
  \text{s.t.}& \quad \Tr(\K(\bm{e}_i-\bm{e}_j)(\bm{e}_i - \bm{e}_j)^T)  \leq u \qquad (i,j) \in \S, \\
  & \quad \Tr(\K (\bm{e}_i-\bm{e}_j)(\bm{e}_i - \bm{e}_j)^T)  \geq \ell \qquad (i,j) \in \D. \nonumber
\end{align}
This kernel learning problem was first proposed in the transductive setting
in~\cite{kulis}, though no extensions to the inductive case were considered.
Note that problem \eqref{eq:itml} optimizes over a $d \times d$
matrix $\A$, while the kernel learning problem~\eqref{eq:itml_kernel} optimizes over an $n \times n$ matrix $\K$.  We now present our key theorem connecting problems \eqref{eq:itml} and \eqref{eq:itml_kernel}.
\begin{theorem}
\label{thm:opt_kernel}
Let $\A^*$ be the optimal solution to problem~\eqref{eq:itml} and let $\K^*$
be the optimal solution to problem~\eqref{eq:itml_kernel}.  Then the optimal
solutions are related by the following:
\begin{eqnarray*}
K^*&=&X^T\A^*X,\\
\A^*&=&I + X M X^T,\\ \mbox{ where } M &=& K_0^{-1}(K^*-K_0) K_0^{-1},\quad K_0=X^TX,\quad X=\left[\phi(\x_1), \phi(\x_2), \dots, \phi(\x_n)\right].
\end{eqnarray*}
  \label{thm:thm2}
\end{theorem}

To prove this theorem, we first prove a lemma for general Bregman matrix
divergences, of which the LogDet divergence is a special case. Consider the following general optimization problem:
\begin{align}
  \min_{\A} &\quad D_{\phi}(W, W_0)\nonumber\\
  \text{s.t.}& \quad \tr(WR_i)\leq s_i,\quad \forall 1\leq i\leq m,\nonumber\\
  & \quad \A\succeq 0,
\label{eq:prob_breg}
\end{align}
where $D_{\phi}(W, W_0)$ is a Bregman matrix divergence~\cite{kulis}
generated by a real-valued strictly convex function over symmetric matrices $\phi:\mathbb{R}^{n\times n}\rightarrow \mathbb{R}$, i.e., 
\begin{equation}\label{eq:breg_div}D_{\phi}(W,
  W_0)=\phi(W)-\phi(W_0)-\tr((W-W_0)^T \nabla \phi(W_0)).\end{equation}
Note that the LogDet divergence is generated by $\phi(\A)=-\log\det\A$. 
\begin{lemma}
\label{lemma:breg_opt}
The solution to the dual of the primal formulation~\eqref{eq:prob_breg} is given by:
\begin{align}
  \max_{\A,\lambda,\Z} &\quad \phi(W)-\phi(W_0)-\tr(W \nabla\phi(W))+\tr(W_0\nabla\phi(W_0))-s(\lambda)\nonumber\\
  \text{s.t.}& \quad \nabla \phi(W)= \nabla \phi(W_0)-R(\lambda)+\Z,\\&\quad \lambda\geq 0,
   \quad \Z\succeq 0,
\label{eq:prob_breg_dual}
\end{align}
where $s(\lambda) = \sum_{i=1}^m \lambda_i s_i$ and $R(\lambda) =
\sum_{i=1}^m \lambda_i R_i$.
\end{lemma}
\begin{proof}
First, consider the Lagrangian of \eqref{eq:prob_breg}:
\begin{align}
  &L(W, \lambda, \Z)=D_{\phi}(W, W_0)+\tr(W R(\lambda))-s(\lambda)-\tr(W\Z),\nonumber\\
  \text{where }&R(\lambda)=\sum_{i=1}^m \lambda_i R_i,\quad 
  s(\lambda)=\sum_{i=1}^m \lambda_i s_i,\quad
  Z \succeq 0, \lambda \geq 0.
  \label{eq:lag_prob_breg}
\end{align}
Now, note that
\begin{equation}
  \label{eq:breg_diff}
\nabla_W D_{\phi}(W, W_0) = \nabla \phi(W)-\nabla \phi(W_0).
\end{equation}
Setting the gradient of the Lagrangian with respect to $W$ to be zero and using~\eqref{eq:breg_diff}, we get:
\begin{align}
  &\nabla \phi(W)-\nabla \phi(W_0)+R(\lambda)-\Z=0,\\
  \text{and so, }&\tr(W \nabla \phi(W_0))=\tr(W \nabla \phi(W))+\tr(W R(\lambda))-\tr(W\Z)\label{eq:berg_A_tr}.
\end{align}
Now, substituting \eqref{eq:berg_A_tr} into the Lagrangian, we get: 
$$L(W, \lambda, \Z)=\phi(W)-\phi(W_0)-\tr(W \nabla \phi(W))+\tr(W_0 \nabla
\phi(W_0)) - s(\lambda),$$
where $\nabla \phi(W)=\nabla \phi(W_0)-R(\lambda)+\Z$. The lemma now follows directly.
\end{proof}
To prove Theorem~\ref{thm:opt_kernel}, we will also need the following well-known lemma:
\begin{lemma}
 $\det(I+AB)=\det(I+BA)$ for all $A\in \mathbb{R}^{m\times n}$, $B\in \mathbb{R}^{n\times m}$.
  \label{lemma:lem1}
\end{lemma}
We are now ready to prove Theorem~\ref{thm:opt_kernel}. 
\begin{proof}{\bf of Theorem~\ref{thm:opt_kernel}.}
First we observe that the squared Mahalanobis distances from the constraints
in~\eqref{eq:itml} may be written as 
\begin{eqnarray*}
  d_\A(\bm{x}_i, \bm{x}_j) & = & \tr(\A(\bm{x}_i - \bm{x}_j)(\bm{x}_i -
\bm{x}_j)^T)\\
 & = & \tr(\A\X(\bm{e}_i - \bm{e}_j)(\bm{e}_i - \bm{e}_j)^T X^T).
\end{eqnarray*}

The objective in problem~\eqref{eq:itml}, $\burg(\A,I)$, is defined only for positive definite $\A$ and is a convex function of $\A$, hence using
Slater's optimality condition, $\Z=0$ (in Lemma~\ref{lemma:breg_opt}) and may be removed from the constraints. Further, note that the LogDet divergence $\burg(\cdot,\cdot)$ is a Bregman matrix divergence with generating function $\phi(\A)=-\log\det(\A)$.  Thus using $\nabla \phi(\A)=-W^{-1}$ and Lemma~\ref{lemma:breg_opt}, the dual of problem \eqref{eq:itml} is given by:
\begin{align}
  \min_{\A,\lambda} &\quad \log\det\A+b(\lambda)\nonumber\\
  \text{s.t.}& \quad \A^{-1}=I+\X\C(\lambda)\X^T,\label{eq:ldA}\\&\quad \lambda\geq 0,\nonumber
\end{align}
where $\C(\lambda)=\sum_{(i,j)\in {\mathcal S}}\lambda_{ij}(\e_i-\e_j)(\e_i-\e_j)^T-\sum_{(i,j)\in {\mathcal D}}\lambda_{ij}(\e_i-\e_j)(\e_i-\e_j)^T$ and $\b(\lambda)=\sum_{(i,j)\in {\mathcal S}}\lambda_{ij}u-\sum_{(i,j)\in {\mathcal D}}\lambda_{ij}\ell$. 

Now, for matrices $\A$ feasible for problem \eqref{eq:ldA}, $ \log\det\A=- \log\det\A^{-1}=-\log\det(I+\X\C(\lambda)\X^T)=-\log\det(I+\C(\lambda)\K_0)$, where the last equality follows from Lemma~\ref{lemma:lem1} (recall that $K_0=X^TX$). Since, $\log\det(AB)=\log\det A+\log \det B$ for square matrices $A$ and $B$, \eqref{eq:ldA} may be rewritten as
\begin{align}
  \min_{\lambda} &\quad -\log\det(\K_0^{-1}+\C(\lambda))+b(\lambda),\nonumber\\
  \text{s.t.}& \quad \lambda\geq 0.
\label{eq:prob4}
\end{align}

Writing $\K^{-1}=\K_0^{-1}+\C(\lambda)$, the above can be written as:
\begin{align}
  \min_{K,\lambda} &\quad \log\det\K+b(\lambda),\nonumber\\
  \text{s.t.}& \quad \K^{-1}=\K_0^{-1}+\C(\lambda), \lambda\geq 0.
\label{eq:prob5}
\end{align}
The above problem can be seen by inspection to be identical to the dual problem of \eqref{eq:itml_kernel} as given by Lemma~\ref{lemma:breg_opt}. Hence, since their dual problems are identical, problems ~\eqref{eq:itml} and~\eqref{eq:itml_kernel} are equivalent.  Using~\eqref{eq:ldA} and the Sherman-Morrison-Woodbury formula, the form of the optimal $\A^*$ is:
$$\A^*=I-\X(\C(\lambda^*)^{-1}+\K_0)^{-1}\X^T = I + \X M \X^T,$$
where $\lambda^*$ is the dual optimal and $M=-(\C(\lambda^*)^{-1}+\K_0)^{-1}$. Similarly, using \eqref{eq:prob5}, the optimal $\K^*$ is given by:
$$\K^*=\K_0-\K_0(\C(\lambda^*)^{-1}+\K_0)^{-1}\K_0=\X^T\A^*\X$$
We can explicitly solve for $M$ as $M  = K_0^{-1}(K^*-K_0)K_0^{-1}$ by
simplification of these expressions using the fact that $K_0 = X^T X$.
This proves the theorem. 
\end{proof}
We now generalize the above theorem to regularize against arbitrary positive definite matrices $\A_0$. 
\begin{corollary}
Consider the following problem:
\begin{equation}
   \label{eq:itml_gen}
  \begin{split}
    \min_{\A \succeq 0} \quad & \burg(\A, \A_0)\\
    \text{s.t.}\quad & d_{\A}(\bm{x}_i,\bm{x}_j) \leq u \qquad (i,j) \in \S,\\ 
    & d_{\A}(\bm{x}_i,\bm{x}_j) \geq \ell \qquad\ (i,j) \in \D.
  \end{split}
\end{equation}
Let $\A^*$ be the optimal solution to problem~\eqref{eq:itml_gen} and let $\K^*$ be the optimal solution to problem~\eqref{eq:itml_kernel}. Then the optimal solutions are related by the following:
\begin{eqnarray*}
K^*&=&X^T\A^*X\\
\A^*&=&\A_0 + \A_0X M X^T\A_0,\\ \mbox{ where } M &=& K_0^{-1}(K^* - K_0) K_0^{-1},\quad K_0=X^T\A_0X,\quad X=\left[\phi(\x_1), \phi(\x_2), \dots, \phi(\x_n)\right]
\end{eqnarray*}
\end{corollary}
\begin{proof}
Note that $\burg(\A, \A_0)=\burg(\A_0^{-1/2}\A\A_0^{-1/2}, I)$. Let $\AA=\A_0^{-1/2}\A\A_0^{-1/2}$. Problem~\eqref{eq:itml_gen} is now equivalent to:
 \begin{equation}
    \label{eq:itml_gen1}
   \begin{split}
     \min_{\AA \succeq 0} \quad & \burg(\AA, I)\\
     \text{s.t.}\quad & d_{\AA}(\xx_i,\xx_j) \leq u \qquad (i,j) \in \S,\\ 
     & d_{\AA}(\xx_i,\xx_j) \geq \ell \qquad\ (i,j) \in \D,
   \end{split}
 \end{equation}
where $\AA=\A_0^{-1/2}\A\A_0^{-1/2}$, $\XX=\A_0^{1/2}X$ and $\XX=[\xx_1, \xx_2, \dots, \xx_n]$. Now using Theorem~\ref{thm:opt_kernel}, the optimal solution $\AA^*$ of problem~\eqref{eq:itml_gen1} is related to the optimal $K^*$ of problem~\eqref{eq:itml_kernel} by $K^*=\XX^T\AA^*\XX=X^T\A_0^{1/2}\A_0^{-1/2}\A^*\A_0^{-1/2}\A_0^{1/2}X=X^T\A^* X$. Similarly, $\A^*=\A_0^{1/2}\AA^*\A_0^{1/2}=\A_0+\A_0XMX^T\A_0$ where $M=K_0^{-1}(K^* - K_0) K_0^{-1}$.
\end{proof}
Since the kernelized version of LogDet metric learning can
be posed as a linearly constrained optimization problem with a LogDet
objective, similar algorithms can be used to solve either problem.  
This equivalence implies that we can \textit{implicitly} solve the metric
learning problem by instead solving for the optimal kernel matrix $K^*$.  
Note that using LogDet divergence as objective function has two significant benefits over many other popular loss functions: 1) the metric and kernel learning problems \eqref{eq:itml}, \eqref{eq:itml_kernel} are both equivalent and hence solving the kernel learning formulation directly provides an out of sample extension (see Section~\ref{sec:generalize_points} for details), 2) projection with respect to the LogDet divergence onto a single distance constraint has a closed form solution, thus making it amenable to an efficient cyclic projection algorithm (refer to Section~\ref{sec:itml_alg}). 

\subsection{Generalizing to New Points}
\label{sec:generalize_points}
In this section, we see how to generalize to new points using the learned kernel matrix $K^*$.

Suppose that we have solved the kernel learning problem for $K^*$ (from now
on, we will drop the $^*$ superscript and assume that $K$ and $\A$ are at optimality).  The
distance between two points $\phi(\bm{x}_i)$ and 
$\phi(\bm{x}_j)$ that are in the training set can be computed directly from
the learned kernel matrix as $K(i, i) + K(j, j) - 2K(i, j)$.  We now consider
the problem of computing the learned distance between two points
$\phi(\bm{z}_1)$ and $\phi(\bm{z}_2)$ that may not be in the training set.

In Theorem~\ref{thm:opt_kernel}, we showed that the optimal solution to the
metric learning problem can be 
expressed as $\A = I + \X M \X^T$. To compute the Mahalanobis distance in
kernel space, we see that the inner product $\phi(\bm{z}_1)^T\A
\phi(\bm{z}_2)$ can be computed entirely via inner products between points:
\begin{eqnarray}
\phi(\bm{z}_1)^T\A \phi(\bm{z}_2) &=& \phi(\bm{z}_1)^T(I+\X M \X^T)\phi(\bm{z}_2) \nonumber\\
&=& \phi(\bm{z}_1)^T\phi(\bm{z}_2) +  \phi(\bm{z}_1)^T\X M \X^T\phi(\bm{z}_2) \nonumber \\
&=& \kappa(\bm{z}_1,\bm{z}_2) + \bm{k}_1^T M \bm{k}_2, \mbox{where } \bm{k}_i
= [\kappa(\bm{z}_i,\bm{x}_1), ..., \kappa(\bm{z}_i, \bm{x}_n)]^T. \label{eq:newpoints}
\end{eqnarray}
Thus, the
expression above can be used to evaluate kernelized distances with
respect to the learned kernel function between arbitrary data objects.


In summary, the connection between kernel learning and metric learning allows us to generalize our metrics to new points in kernel space.
This is performed by first solving the kernel learning problem for $K$, then
using the learned kernel matrix and the input kernel function to compute
learned distances via~\eqref{eq:newpoints}. 
 
\subsection{Kernel Learning Algorithm}
\label{sec:itml_alg}
Given the connection between the Mahalanobis metric learning problem for the
$d \times d$ matrix $\A$ and the kernel learning problem for the $n \times n$
kernel matrix $K$, we would like to develop an algorithm for efficiently
performing metric learning in kernel space.
Specifically, we provide an algorithm (see Algorithm~\ref{algo:burg}) for solving the kernelized LogDet
metric learning problem, as given in \eqref{eq:itml_kernel}.  

First, to avoid problems with infeasibility, we incorporate
\textit{slack variables} into our formulation.  These provide a tradeoff
between minimizing the divergence between $\K$ and $\K_0$ and satisfying the
constraints.  Note that our earlier results (see Theorem~\ref{thm:opt_kernel}) easily generalize to the slack case:
\begin{equation}
   \label{eqn:burgslackobj}
  \begin{split}
    \min_{\K, \bm{\xi}} \quad & \burg(\K, \K_0) + \gamma \cdot \burg(\mbox{diag}(\bm{\xi}),\mbox{diag}(\bm{\xi}_0))\\
    \text{s.t.}\quad & \Tr(\K (\bm{e}_i - \bm{e}_j)(\bm{e}_i - \bm{e}_j)^T) \leq \xi_{ij} \quad (i,j)
    \in \S,\\ 
    & \Tr(\K (\bm{e}_i - \bm{e}_j)(\bm{e}_i - \bm{e}_j)^T) \geq \xi_{ij}
    \quad (i,j) \in \D.
  \end{split}
\end{equation}
The parameter $\gamma$ above controls the tradeoff between satisfying the
constraints and minimizing $\burg(\K,\K_0)$, and the entries of $\bm{\xi}_0$
are set to be $u$ for corresponding similarity constraints and $\ell$ for
dissimilarity constraints.

\renewcommand{\algorithmicrequire}{\textbf{Input:}}
\renewcommand{\algorithmicensure}{\textbf{Output:}}
\renewcommand{\algorithmicrepeat}{\hspace*{-1em} 3.~\textbf{repeat}}
\renewcommand{\algorithmicuntil}{{\hskip -1em}\textbf{until}}
\renewcommand{\algorithmicreturn}{{\hskip -1em}\textbf{return}}

\begin{algorithm}[t]\small
  \centering
  \caption{Metric/Kernel Learning with the LogDet Divergence\label{algo:burg}}
  \begin{algorithmic}
    \begin{minipage}{.9\textwidth}
      \REQUIRE $\K_0$: input $n \times n$ kernel matrix, $\S$: set of similar
      pairs, $\D$: set of dissimilar pairs, $u, \ell$: distance thresholds,
      $\gamma$: slack parameter
      \ENSURE  {$\K$: output kernel matrix}
      \vspace*{5pt}
      \STATE \hspace*{-1em} 1.~$\K \gets \K_0$, $\lambda_{ij} \gets 0~ \forall~ij$
      \STATE \hspace*{-1em} 2.~$\xi_{ij} \gets u$ for $(i,j) \in \S$; otherwise $\xi_{ij} \gets \ell$
      \STATE \hspace*{-1em} 3.~\textbf{repeat}
      \vspace*{-6pt}
      \renewcommand{\labelenumi}{3.\theenumi.}
      \begin{enumerate}
        \setlength{\itemsep}{-2pt}
      \item Pick a constraint $(i,j) \in \S$ or $\D$
      \item $p \gets (\bm{e}_{i} - \bm{e}_{j})^T \K (\bm{e}_{i} - \bm{e}_{j})$
      \item $\delta \gets 1$ if $(i,j) \in \S$, $-1$ otherwise
      \item $\alpha \gets \min \Big (\lambda_{ij},\frac{\delta \gamma}{\gamma+1} \Big ( \frac{1}{p} - \frac{1}{\xi_{ij}} \Big ) \Big )$
      \item $\beta \gets \delta \alpha / (1 - \delta \alpha p)$
      \item $\xi_{ij} \gets \gamma \xi_{ij} / (\gamma + \delta \alpha \xi_{ij})$
      \item $\lambda_{ij} \gets \lambda_{ij} - \alpha$
      \item $\K \leftarrow \K + \beta \K (\bm{e}_{i} - \bm{e}_{j})(\bm{e}_{i} - \bm{e}_{j})^T \K$
      \end{enumerate}
      \vspace*{-5pt}
      \STATE \hspace*{-1em} 4.~\textbf{until} {convergence}
      \RETURN $\K$
    \end{minipage}
\end{algorithmic}
\label{alg:itml}
\end{algorithm}
To solve problem \eqref{eqn:burgslackobj}, we employ the technique of \textit{Bregman 
projections}, as discussed in the transductive setting~\cite{kulis,kulis_jmlr}.  At each iteration, we choose a constraint $(i,j)$ from $\S$ or
$\D$.  We then apply a Bregman projection such that $\K$ satisfies
the constraint after projection; note that the projection is not an orthogonal
projection but is rather tailored to the particular function that we are
optimizing.
Algorithm~\ref{algo:burg} details the steps for Bregman's method on this
optimization problem.  Each update is given by a rank-one update
\begin{displaymath}
\K \leftarrow \K + \beta \K (\bm{e}_{i} - \bm{e}_{j})(\bm{e}_{i} -
\bm{e}_{j})^T \K,
\end{displaymath}
where $\beta$ is an appropriate projection parameter that can be computed in closed form (see Algorithm~\ref{algo:burg}).

Algorithm~\ref{algo:burg} has a number of key properties which make it
useful for various kernel learning tasks.  First, the Bregman projections
can be computed in closed form, assuring that the projection updates are
efficient ($O(n^2)$).  Note that, if the feature space dimensionality $d$ is less than $n$ then a similar algorithm can be used directly in the feature space (see \cite{itml}). Instead of LogDet, if we use the von Neumann divergence, another potential
loss function for this problem, $O(n^2)$
updates are possible, but are much more complicated and require use of the
fast multipole method, which cannot be employed easily in
practice.  Secondly, the projections maintain positive definiteness, which
avoids any eigenvector computation or semidefinite programming.  This is in stark contrast with the Frobenius loss, which requires additional computation to
maintain positive definiteness, leading to $O(n^3)$ updates.
\subsection{Metric/Kernel Learning with Large Datasets}
\label{sec:itml_large}
In Sections~\ref{sec:itml_prob} and \ref{sec:itml_kernel} we proposed a LogDet divergence based Mahalanobis metric learning problem~\eqref{eq:itml} and an equivalent kernel learning problem~\eqref{eq:itml_kernel}. The number of parameters involved in these problems is $O(\min(n^2,d^2))$, where $n$ is the number of training points and $d$ is the dimensionality of the data. This quadratic dependency effects not only the running time for both training and testing, but also poses tremendous challenges in
estimating a quadratic number of parameters.  For example, a data set with
10,000 dimensions leads to a Mahalanobis matrix with 100
million values.  This represents a fundamental limitation of existing
approaches, as many modern data mining problems possess relatively high
dimensionality.   

In this section, we present a method for learning structured Mahalanobis
distance (kernel) functions that scale linearly with the dimensionality (or training set size).  Instead of
representing the Mahalanobis distance/kernel matrix as a full $d \times d$ (or $n\times n$) matrix with $O(\min(n^2,d^2))$ parameters, our methods use compressed representations, admitting
matrices parameterized by $O(\min(n,d))$ values.  This enables the Mahalanobis
distance/kernel function to be learned, stored, and evaluated efficiently in the context of high dimensionality and large training set size. In particular, we propose a method to efficiently learn an identity plus low-rank Mahalanobis distance matrix and its equivalent kernel function. 

Now, we formulate the high-dimensional identity plus low-rank (\dlr) metric learning problem. Consider a low-dimensional subspace in $\mathbb{R}^d$ and let the columns of $U$ form an orthogonal basis of this subspace. 
 We will constrain the learned Mahalanobis distance matrix to be of the form:
\begin{equation}
  \label{eq:iplr_mahal}
  \A=I^d+\A_l=I^d+ULU^T,
\end{equation}
where $I^d$ is the $d\times d$ identity matrix, $\A_l$ denotes the low-rank
part of $\A$ and $L\in\mathbb{S}_+^{k\times k}$ with $k\ll \min(n,d)$.  Analogous to~\eqref{eq:itml}, we propose the following problem to learn an identity plus low-rank Mahalanobis distance function:
\begin{equation}
   \label{eq:iplr}
  \begin{split}
    \min_{\A,L \succeq 0} \quad & \burg(\A, I^d)\\
    \text{s.t.}\quad & d_\A(\bm{x}_i,\bm{x}_j) \leq u \qquad (i,j) \in {\mathcal S},\\ 
    & d_\A(\bm{x}_i,\bm{x}_j) \geq \ell \qquad\ (i,j) \in {\mathcal D},\\
    & \A=I^d+ULU^T.
  \end{split}
\end{equation}
Note that the above problem is identical to \eqref{eq:itml} except for the added constraint $\A=I^d+ULU^T$. 

Let $F=I^k+L$. Now we have
\begin{align}
\burg(\A,I^d)&=\tr(I^d+ULU^T)-\log\det(I^d+ULU^T)-d,\nonumber\\
&=\tr(I^k+L)+d-k-\log\det(I^k+L)-d,\nonumber\\
&=\burg(F,I^k),
\label{eq:newburg}
\end{align}
where the second equality follows from the fact that $\tr(AB)=\tr(BA)$ and Lemma~\ref{lemma:lem1}. 
Also note that for all $\C\in\mathbb{R}^{n\times n}$,
\begin{align}
\tr(\A\X C\X^T)&=\tr((I^d+ULU^T)\X C\X^T),\nonumber\\
&=\tr(\X C\X^T)+\tr(LU^T\X C\X^TU),\nonumber\\
&=\tr(\X C\X^T)-\tr(\X'C{\X'}^T)+\tr(F\X'C{\X'}^T),\nonumber
\end{align}
where $\X'=U^T\X$ is the reduced-dimensional representation of $X$. Hence, 
\begin{equation}
  \label{eq:newdis}
  d_\A(\x_i,\x_j)=\tr(\A\X(\e_i-\e_j)(\e_i-\e_j)^T\X^T)=d_I(\x_i,\x_j)-d_I(\x'_i,\x'_j)+d_F(\x_i',\x_j').
\end{equation}
Using \eqref{eq:newburg} and \eqref{eq:newdis}, problem \eqref{eq:iplr} is equivalent to the following:
\begin{equation}
   \label{eq:iplr1}
  \begin{split}
    \min_{F \succeq 0} \quad & \burg(F, I^k)\\
    \text{s.t.}\quad & d_F(\bm{x'}_i,\bm{x'}_j) \leq u-d_I(\x_i,\x_j)+d_I(\x'_i,\x'_j) \qquad (i,j) \in {\mathcal S},\\ 
    & d_F(\bm{x'}_i,\bm{x'}_j) \geq\ell-d_I(\x_i,\x_j)+d_I(\x'_i,\x'_j) \qquad\ (i,j) \in {\mathcal D}.
  \end{split}
\end{equation}
Note that the above formulation is an instance of problem~\eqref{eq:itml} and can be solved using an algorithm similar to Algorithm~\ref{alg:itml}. Furthermore, the above problem solves for a $k\times k$ matrix rather than a $d\times d$ matrix seemingly required by \eqref{eq:iplr}. The optimal $\A^*$ is obtained as $\A^*=I^d+U(F^*-I^k)U^T$.

Next, we show that problem \eqref{eq:iplr1} and equivalently \eqref{eq:iplr} can be solved efficiently in feature space by selecting an appropriate basis $R$ ($U=R(R^TR)^{-1/2}$). Let $R=XJ$, where $J\in\mathbb{R}^{n\times k}$. Note that $U=XJ(J^TK_0J)^{-1/2}$ and $X'=U^TX=(J^TK_0J)^{-1/2}J^TK_0$, i.e., $X'\in \mathbb{R}^{k\times n}$ can be computed efficiently in the feature space (requiring inversion of only a $k\times k$ matrix). Hence, problem \eqref{eq:iplr1} can be solved efficiently in feature space using Algorithm~\ref{alg:itml} and the optimal kernel $K^*$ is given by $$K^*=X^T\A^*X=K_0+K_0J(J^TK_0J)^{-1/2}(F^*-I^k)(J^TK_0J)^{-1/2}J^TK_0.$$ 

Note that problem \eqref{eq:iplr1} can be solved via Algorithm~\ref{alg:itml} using $O(k^2)$ computational steps per iteration. Additionally, $O(\min(n,d)k)$ steps are required to prepare the data. Also, the optimal solution $\A^*$ (or $K^*$) can be stored implicitly in $O(\min(n,d)k)$ steps and similarly, the Mahalanobis distance between any two points can be computed in time $O(\min(n,d)k)$ steps. 

The metric learning problem presented here depends critically on the basis selected. For the case when $d$ is not significantly larger than $n$ and feature space vectors $X$ are available explicitly, the basis $R$ can be selected by using one of the following heuristics (see Section 5, \cite{jasonkdd} for more  details): 
\begin{itemize}
\item Using the top $k$ singular vectors of $X$.
\item Clustering the columns of $X$ and using the mean vectors as the basis $R$.
\item For the fully-supervised case, if the number of classes ($c$) is
  greater than the required dimensionality ($k$) then cluster the class-mean
  vectors into $k$ clusters and use the obtained cluster centers for forming
  the basis $R$. If $c<k$ then cluster each class into $k/c$ clusters and use
  the cluster centers to form $R$. 
\end{itemize}
For learning the kernel function, the basis $R=XJ$ can be selected by: 1)
using a randomly sampled coefficient matrix $J$, 2) clustering $X$ using
kernel $k$-means or a spectral clustering method, 3) choosing a random subset
of $X$, i.e, the columns of $J$ are random indicator vectors. A more careful selection of the basis $R$ should further improve accuracy of our method and is left as a topic for future research. 


\renewcommand{\S}{S}
\section{Kernelization with Other Convex Loss Functions}\label{sec:loss}

One of the key benefits to using the LogDet divergence for metric learning is
 its ability to efficiently learn a linear mapping for high-dimensional kernelized data.  A natural
question is whether one can kernelize metric learning with other loss
functions, such as those considered previously in the literature.  To this
end, the work
of~\cite{arxiv} showed how to kernelize some popular metric learning
algorithms such as MCML~\cite{roweis} and LMNN~\cite{wein}.  In this section,
we show a complementary result that shows how to kernelize a class
of metric learning algorithms that learns a linear map in input or feature space.

Consider the following (more) general optimization problem that may be viewed as a generalization of \eqref{eq:itml} for learning a linear transformation matrix $G$, where $\A=G^TG$:
\begin{align}
  \min_{\A} &\quad \tr(f(\A))\nonumber\\
  \text{s.t.}& \quad \tr(\A\X\C_i\X^T)\leq \b_i,\quad \forall 1\leq i\leq m\nonumber\\
  & \quad \A\succeq 0,
\label{eq:prob1}
\end{align}
where $f:\mathbb{R}^{d\times d}\rightarrow \mathbb{R}^{d\times d}$, $\tr(f(\A))$ is a convex function, $\A\in S_{+}^{d\times d}$, $\X\in \mathbb{R}^{d\times n}$, and each $\C_i\in
\mathbb{R}^{n\times n}$ is a symmetric matrix.  Note that we have generalized
both the loss function and the constraints.  For example, the LogDet
divergence can be viewed as a special case, since we may write $D_{\ell
  d}(X,Y) = \tr(XY^{-1} - \log (XY^{-1}) - I)$. The loss function $f(\A)$ regularizes the learned transformation $\A$ against the baseline Euclidean distance metric, i.e., $\A_0=I$. Hence, a desirable property of $f$ would be: $\tr(f(\A))\geq 0$ with $\tr(f(\A))=0$ iff $\A=I$. 

In this section we show that for a large and important class of functions $f$, problem \eqref{eq:prob1} can be solved for $W$ implicitly in the feature space, i.e., the problem \eqref{eq:prob1} is \textit{kernelizable}. We assume that the kernel function $K_0(\x,\y)=\phi(\x)^T\phi(\y)$ between any two data points can be computed in $O(1)$ time. Denote $W^*$ as an optimal solution for \eqref{eq:prob1}. Now, we formally define \textit{kernelizable} metric learning problems. 
\begin{definition}
\label{def:kernelml}
An instance of metric learning problem~\eqref{eq:prob1} is \textit{kernelizable} if the following conditions hold:
\begin{itemize}
\item Problem~\eqref{eq:prob1} is solvable efficiently in time poly($n$, $m$) without explicit use of feature space vectors $X$. 
\item $\tr(\A^*Y\C Y^T)$, where $Y\in \mathbb{R}^{d\times N}$ is the feature space representation of any given data points, can be computed in time poly($N$) for all $C\in \mathbb{R}^{N\times N}$. 
\end{itemize}
\end{definition}

\begin{theorem}
  Let $f:\mathbb{R}\rightarrow \mathbb{R}$ be a function defined over the reals such that:
\begin{itemize}
\item $f(x)$ is a convex function.
\item A sub-gradient of $f(x)$ can be computed efficiently in $O(1)$ time. 
\item  $f(x)\geq 0\ \forall x$ with $f(\eta)=0$ for some $\eta\geq 0$. 
\end{itemize}
Consider the extension of $f$ to the spectrum of $\A\in \S_d^+$, i.e. $f(\A)=Uf(\Lambda)U^T$, where $\A=U\Lambda U^T$ is the eigenvalue decomposition of $\A$ (Definition 1.2, \cite{higham}). Assuming $X$ to be full-rank, i.e., $K_0=X^TX$ is invertible, problem \eqref{eq:prob1} is kernelizable (Definition~\ref{def:kernelml}). 
\label{thm:thm1}
\end{theorem}
To prove the above theorem, we need the following two lemmas:
\begin{lemma}
Assuming $f$ satisfies the conditions stated in Theorem~\ref{thm:thm1} and $X$ is full-rank, $\exists S^*\in \mathbb{R}^{n\times n}$ such that $\A^*=\eta I+\X\S^*\X^T$ is an optimal solution to \eqref{eq:prob1}.
\label{lemma:thm1_1}
\end{lemma}
\begin{proof}
Let $\A=U\Lambda U^T=\sum_j \lambda_j \bm{u}_j\bm{u}_j^T$ be the eigenvalue
decomposition of $\A$, where $\lambda_1\geq \lambda_2\geq \dots\geq \lambda_d\geq 0$. Consider a linear constraint $\tr(\A \X\C_i\X^T)\leq
b_i$ as specified in problem \eqref{eq:prob1}. Note that $\tr(\A
\X\C_i\X^T)=\sum_j\lambda_j \bm{u}_j^T\X\C_i\X^T\bm{u}_j$. Note that if the
$j$-th eigenvector $\bm{u}_j$ of $\A$ is orthogonal to the range space of
$X$, i.e. $X^T\bm{u}_j=0$, then the corresponding eigenvalue $\lambda_j$ is
not constrained (except for the non-negativity constraint imposed by the
positive semi-definiteness constraint). Since the range space of $X$ is at
most $n$-dimensional, without loss of generality we can assume that
$\lambda_j\geq 0, \forall j> n$ are not constrained by the linear inequality constraints in \eqref{eq:prob1}.

Furthermore, by the definition of a spectral function (Definition 1.2, \cite{higham}), $\tr(f(\A))=\sum_j f(\lambda_j)$. Since $f$ satisfies the conditions of Theorem~\ref{thm:thm1}, $f(\eta)=\min_x f(x)=0$. In order to minimize $\tr(f(\A))$, we can select  $\lambda_j^*=\eta\geq 0, \forall j> n$ (note that the non-negativity constraint is satisfied for this choice of $\lambda_j$). 
Furthermore, eigenvectors $\bm{u}_j,\ \forall j\leq n$, lie in the range space of $X$, i.e., $\forall j\leq n,\ \bm{u}_j=X\bm{\alpha}_j $ for some $\bm{\alpha}_j\in \mathbb{R}^n$. Hence, 
\begin{align*}
\A^*&=\sum_{j=1}^n\lambda_i^*\bm{u}_j^*\bm{u}_j^{*T}+\eta\sum_{j=n+1}^d \bm{u}_j^*\bm{u}_j^{*T},\\
&=\sum_{j=1}^n(\lambda_i^*-\eta)\bm{u}_j^*\bm{u}_j^{*T} +\eta\sum_{j=1}^d \bm{u}_j^*\bm{u}_j^{*T},\\
&=\sum_{j=1}^n\X((\lambda_j^*-\eta) \bm{\alpha}_j^*\bm{\alpha}_j^{*T})\X^T+\eta I^d,\\
&=\X S^*\X^T+\eta I^d,
\end{align*}
where $S^*=\sum_{j=1}^n(\lambda_j^*-\eta) \bm{\alpha}_j^*\bm{\alpha}_j^{*T}$.
\end{proof}

\begin{lemma}
If $n<d$ and $\X\in \mathbb{R}^{d\times n}$ has full column rank, i.e., $\X^T\X$ is invertible then:
$$\X\S\X^T\succeq 0\ \  \Longleftrightarrow\ \  \S\succeq 0.$$
  \label{lemma:lem3}
\end{lemma}
\begin{proof}
$\Longrightarrow$\\ $\X\S\X^T\succeq 0 \Longrightarrow \v^T \X\S\X^T \v\geq 0, \forall \v\in \mathbb{R}^d$.
Since $\X$ has full column rank, $\forall \q\in \mathbb{R}^n\ \exists \bm{v}\in \mathbb{R}^d $ s.t. $X^T\bm{v} =\q$. Hence, $\q^T\S\q=\v^T\X\S\X^T\v\geq 0, \forall \q\in \mathbb{R}^n \Longrightarrow \S\succeq 0$\\
$\Longleftarrow$\\
Now $\forall \v\in \mathbb{R}^d$, $\v^TXSX^T\v\geq 0$ as $S\succeq 0$. Thus $\X\S\X^T\succeq 0$.
\end{proof}

We now present a proof of Theorem~\ref{thm:thm1}. The key idea is to prove that \eqref{eq:prob1} can solved  implicitly by solving for $S^*$ of Lemma~\ref{lemma:thm1_1}. 
\begin{proof}{\bf [Theorem~\ref{thm:thm1}]}\\
Using Lemma~\ref{lemma:thm1_1}, $\A^*$ is of the form $\A^*=\eta I^d+\X
S^*\X^T$. Assuming $\X$ is full-rank, i.e., all the data points $\x_i$ are
linearly independent, then there is a one-to-one mapping between $\A^*$ and $S^*$. Hence, solving for $\A^*$ is \textit{equivalent} to solving for $S^*$. So, now our goal is to reformulate problem~\eqref{eq:prob1} in terms of $S^*$. 

\noindent Let $\X=U_X\Sigma_X V_X^T$ be the SVD of $X$. Then, 
\begin{align}
\A&=\eta I^d+\X S \X^T\nonumber,\\
&=\eta I^d+U_X\Sigma_X V_X^T S V_X \Sigma_X U_X^T\nonumber,\\
&=\left[U_X\ \  U_\perp\right]\left[\begin{matrix}\Sigma_X V_X^T S V_X \Sigma_X+\eta I^n &\ \  0\\[4pt]0&\eta I^{n-d}\end{matrix}\right]\left[\begin{matrix}U_X^T\\[4pt]U_\perp^T\end{matrix}\right],
\label{eq:prob1A}
\end{align}
where $U_\perp^TU=0$. 

Now, consider $f(\A)=f(\eta I^d+\X S \X^T)$. Using \eqref{eq:prob1A}:
\begin{align}
f(\A)&=f(\eta I^d+\X S \X^T)\nonumber,\\[4pt]
&=f\left(\left[U_X\ \  U_\perp\right]\left[\begin{matrix}\Sigma_X V_X^T S V_X \Sigma_X+\eta I^n &\ \  0\\[4pt]0&\eta I^{n-d}\end{matrix}\right]\left[\begin{matrix}U_X^T\\[4pt]U_\perp^T\end{matrix}\right]\right)\nonumber,\\[4pt]
&=\left[U_X\ \  U_\perp\right]f\left(\left[\begin{matrix}\Sigma_X V_X^T S V_X \Sigma_X+\eta I^n &\ \  0\\[4pt]0&\eta I^{n-d}\end{matrix}\right]\right)\left[\begin{matrix}U_X^T\\[4pt]U_\perp^T\end{matrix}\right]\nonumber,\\[4pt]
&=\left[U_X\ \  U_\perp\right]\left[\begin{matrix}f\left(\Sigma_X V_X^T S V_X \Sigma_X+\eta I^n\right) &\ \  0\\[4pt]0&0\end{matrix}\right]\left[\begin{matrix}U_X^T\\[4pt]U_\perp^T\end{matrix}\right]\nonumber,\\[4pt]
&=U_X f\left(\Sigma_X V_X^T S V_X \Sigma_X+\eta I^n\right)U_X^T\nonumber,
\end{align}
where the second equality follows from the property that $f(Q Z
Q^T)=Qf(Z)Q^T$ for an orthogonal $Q$ and a spectral function $f$. The third
equality follows from the property that
$f\left(\begin{matrix}A&0\\0&B\end{matrix}\right)=\left(\begin{matrix}f(A)&0\\0&f(B)\end{matrix}\right)$
and the fact that $f(\eta )=0$. Hence,  
\begin{equation}
\tr(f(\A))=f\left(\Sigma_X V_X^T S V_X \Sigma_X+\eta I^n\right).
  \label{eq:prob1f}
\end{equation}

Next, consider the constraint $\tr(\A\X\C_i\X^T)\leq b_i$. Note that 
\begin{equation}
\tr(\A\X\C_i\X^T)=\tr((\eta I^d+\X\S\X^T)\X\C_i\X^T)=\tr(\eta C_iK_0+C_iK_0\S K_0).
\end{equation}
Hence, the constraint $\tr(\A\X\C_i\X^T)\leq b_i$ reduces to:
\begin{equation}
\tr(\eta C_iK_0+C_iK_0\S K_0)\leq b_i.
\label{eq:prob1cons}
\end{equation}

Finally, consider the constraint $\A\succeq 0$. Using ~\eqref{eq:prob1A}, we see that this is equivalent to:
\begin{align}
  \eta I^n+\Sigma_X V_X^T S V_X \Sigma_X \succeq 0, \nonumber\\
  S \succeq - \eta K_0^{-1},
\label{eq:prob1psd}
\end{align}
where  $K_0=X^TX=V_X\Sigma_X^2 V_X^T$. 

Using \eqref{eq:prob1f}, \eqref{eq:prob1cons}, and \eqref{eq:prob1psd} we get the following problem which is equivalent to \eqref{eq:prob1}:
\begin{align}
  \min_{\S} &\quad f\left(\Sigma_X V_X^T S V_X \Sigma_X+\eta I^n\right)\nonumber\\
  \text{s.t.}& \quad \tr(\eta C_iK_0+C_iK_0\S K_0)\leq b_i,\quad \forall 1\leq i\leq m\nonumber\\
  & \quad S \succeq -\eta K_0^{-1}.
\label{eq:prob1new}
\end{align}
Note that the objective function is a strictly convex function of a linear
transformation of $S$, and hence is strictly convex in $S$. Furthermore, all
the constraints are linear in $S$.  As a result, problem \eqref{eq:prob1new} is a convex program. Also, both $\Sigma_X$ and $V_X$ can be computed efficiently in $O(n^3)$ steps using eigenvalue decomposition of $K_0=X^TX$. Hence, problem \eqref{eq:prob1} can be  solved efficiently in poly($n,m$) steps using standard convex optimization methods such as the ellipsoid method \cite{gla}. 
\end{proof}

\section{Special Cases}
In the previous section, we proved a general result on kernelization of metric learning. In this section, we further
consider a few special cases of interest: the von Neumann divergence, the squared Frobenius norm and semi-definite programming. For each of the cases, we derive the required optimization problem to be solved and mention the relevant optimization algorithms that can be used. 
\subsection{von Neumann Divergence}
The von Neumann divergence is a generalization of the well known KL-divergence to matrices. It is used extensively in quantum computing to compare density matrices of two different systems~\cite{nielsenbook}. It is also used in the exponentiated matrix gradient method by~\cite{tsuda}, online-PCA method by \cite{online-pca} and fast SVD solver by~\cite{kale}. The von Neumann divergence between $\A$ and $\A_0$ is defined to be:
$$\von(\A, \A_0)=\tr(\A\log \A-\A\log \A_0-\A+\A_0),$$
where both $\A$ and $\A_0$ are positive definite. The metric learning problem that corresponds to \eqref{eq:prob1} is:
\begin{align}
  \min_{\A} &\quad \von(\A, I)\nonumber\\
  \text{s.t.}& \quad \tr(\A\X\C_i\X^T)\leq \b_i,\quad \forall 1\leq i\leq m,\nonumber\\
  & \quad \A\succeq 0.
\label{eq:prob6}
\end{align}
It is easy to see that $\von(\A, I)=\tr(\fvN(\A))$, where $$\fvN(\A)=\A\log \A-\A+I=U\fvN (\Lambda)U^T,$$ where $\A=U\Lambda U^T$ is the eigenvalue decomposition of $\A$ and $\fvN:\mathbb{R}\rightarrow \mathbb{R}, \fvN(x)=x \log x - x+1$. Also, note that $\fvN(x)$ is a strictly convex function with $\argmin_x \fvN(x)=1$ and $\fvN(1)=0$. Hence, using Theorem~\ref{thm:thm1}, problem \eqref{eq:prob6} is kernelizable since $\von(\A,I)$ satisfies the required conditions. Using~\eqref{eq:prob1new}, the optimization problem to be solved is given by:
\begin{align}
  \min_{\S} &\quad \von\left(\Sigma_X V_X^T S V_X \Sigma_X+I^n, I^n\right)\nonumber\\
  \text{s.t.}& \quad \tr( C_iK_0+C_iK_0\S K_0)\leq b_i,\quad \forall 1\leq i\leq m\nonumber\\
  & \quad S \succeq - K_0^{-1},
\label{eq:vnnew}
\end{align}
Next, we derive a simplified version of the above optimization problem. 

Note that $\von(\cdot,\cdot)$ is defined only for positive semi-definite matrices. Hence, the constraint $S \succeq - K_0^{-1}$ should be satisfied if the above problem is feasible. Thus, the reduced optimization problem is given by:
\begin{align}
  \min_{\S} &\quad \von\left(\Sigma_X V_X^T S V_X \Sigma_X+I^n, I^n\right)\nonumber\\
  \text{s.t.}& \quad \tr( C_iK_0+C_iK_0\S K_0)\leq b_i,\quad \forall 1\leq i\leq m.
\label{eq:vnnew1}
\end{align}
Note that the von-Neumann divergence is a Bregman matrix divergence (see Equation~\eqref{eq:breg_div}) with the generating function $\phi(\X)=\tr(\X\log\X-\X)$. Now using Lemma~\ref{lemma:breg_opt} and simplifying using the fact that $\frac{\partial \tr(\X\log\X)}{\partial \X}=\log\X$, we get the following dual for problem~\eqref{eq:prob6}:
\begin{align}
  \max_{\lambda} &\quad -\tr(\exp(-\Sigma_X V_X^T \C(\lambda) V_X \Sigma_X))-b(\lambda)\nonumber\\
  \text{s.t.}& \quad \lambda\geq 0,
\end{align}
where $\C(\lambda)=\sum_i\lambda_i\C_i$ and $\b(\lambda)=\sum_i\lambda_i\b_i$. 

Now, using $V_X \Sigma_X^2 V_X^T=K_0$  we see that: $\tr(-\Sigma_X V_X^T \C(\lambda) V_X \Sigma_X)^k)=\tr((- \C(\lambda) K_0)^k)$. Next, using the Taylor series expansion for the matrix exponential:
\begin{align*}
\tr(\exp(-\Sigma_X V_X^T \C(\lambda) V_X \Sigma_X))&=\tr\left(\sum_{i=0}^\infty \frac{(-\Sigma_X V_X^T \C(\lambda) V_X \Sigma_X)^i}{i!}\right)\\&=\sum_{i=0}^\infty \frac{\tr\left((-\Sigma_X V_X^T \C(\lambda) V_X \Sigma_X)^i\right)}{i!}\\&=\sum_{i=0}^\infty \frac{\tr\left((- \C(\lambda) K_0)^i\right)}{i!}=\tr(\exp(- \C(\lambda) K_0)).
\end{align*}

Hence, the resulting dual problem is given by: 
\begin{align}
  \min_\lambda&\quad F(\lambda)=\tr(\exp(-\C(\lambda)\K_0))+\b(\lambda)\nonumber\\
  \text{s.t.}&\quad \lambda\geq 0.\label{eq:prob6_dual}
\end{align}
Also, $\frac{\partial F}{\partial \lambda_i}=\tr(\exp(-\C(\lambda)\K_0)\C_i\K_0)+\b_i$. Hence, any first order smooth optimization method can be used to solve the above dual problem. Also, similar to \cite{kulis}, a Bregman's cyclic projection method can be used to solve the primal problem ~\eqref{eq:vnnew1}. 
\subsection{Squared Frobenius Divergence}
The squared Frobenius norm divergence is defined as:
$$\frob(\A, \A_0)=\frac{1}{2}\|\A-\A_0\|_F^2,$$
and is a popular measure of distance between matrices. Consider the following instance of~\eqref{eq:prob1} with the squared Frobenius divergence as the objective function:
\begin{align}
  \min_{\A} &\quad \frob(\A, \eta I)\nonumber\\
  \text{s.t.}& \quad \tr(\A\X\C_iX^T)\leq \b_i,\quad \forall 1\leq i\leq m,\nonumber\\
  & \quad \A\succeq 0.
\label{eq:probfrob}
\end{align}
Note that for $\eta=0$ and $C_i=(\e_a-\e_b)(\e_a-\e_b)^T-(\e_a-\e_c)(\e_a-\e_c)^T$ (relative distance constraints), the above problem \eqref{eq:probfrob} is the same as the one proposed by \cite{shalev}. Below we see that, similar to \cite{shalev}, Theorem~\ref{thm:thm1} in Section~\ref{sec:loss} guarantees kernelization for a more general class of Frobenius divergence based objective functions. 

It is easy to see that $\frob(\A, \eta I)=\tr(\ffrob(\A))$, where $$\ffrob(\A)=(\A-\eta I)^T(\A-\eta I)=U\ffrob(\Lambda)U^T,$$ $\A=U\Lambda U^T$ is the eigenvalue decomposition of $\A$ and $\ffrob:\mathbb{R}\rightarrow \mathbb{R}, \ffrob(x)=(x-\eta)^2$. Note that $\ffrob(x)$ is a strictly convex function with $\argmin_x \ffrob(x)=\eta$ and $\ffrob(\eta)=0$. Hence, using Theorem~\ref{thm:thm1}, problem ~\eqref{eq:prob6} is kernelizable since $\frob(\A,\eta I)$ satisfies the required conditions. Using~\eqref{eq:prob1new}, the optimization problem to be solved is given by:
\begin{align}
  \min_{\S} &\quad \|\Sigma_X V_X^T S V_X \Sigma_X\|_F^2\nonumber\\
  \text{s.t.}& \quad \tr( \eta C_iK_0+C_iK_0\S K_0)\leq b_i,\quad \forall 1\leq i\leq m\nonumber\\
  & \quad S \succeq - \eta K_0^{-1},
\label{eq:frobnew}
\end{align}
Also, note that $\|\Sigma_X V_X^T S V_X \Sigma_X\|_F^2=\tr(K_0SK_0S)$. The above problem can be solved using standard convex optimization techniques like interior point methods. 
\subsection{SDPs}
In this section we consider the case when the objective function in~\eqref{eq:prob1} is a linear function. A similar formulation for metric learning was proposed by~\cite{wein}. We consider the following generic semidefinite program (SDP) to learn a linear transformation $\A$:
\begin{align}
  \min_{\A}\ &\quad \tr(\X\C_0\X^T\A)\nonumber\\
  \text{s.t.}& \quad \tr(\A\X\C_i\X^T)\leq \b_i,\quad \forall 1\leq i\leq m\nonumber\\
  & \quad \A\succeq 0.
\label{eq:prob8}
\end{align}
Here we show that this problem can be efficiently solved for high dimensional data in its kernel space. 
\begin{theorem}
Problem~\eqref{eq:prob8} is kernelizable. 
  \label{thm:thm5}
\end{theorem}
\begin{proof}
\eqref{eq:prob8} has a linear objective, i.e., it is a non-strict convex problem that may have multiple solutions. A variety of regularizations can be considered that lead to slightly different solutions. Here, we consider two regularizations:
\begin{itemize}
\item {\bf Frobenius norm}: We add a squared Frobenius norm regularization to
  \eqref{eq:prob8} so as to find the minimum Frobenius norm solution to
  \eqref{eq:prob8} (when $\gamma$ is sufficiently small):
\begin{align}
  \min_{\A}\ &\quad \tr(\X\C_0\X^T\A)+\frac{\gamma}{2} \|\A\|_F^2\nonumber\\
  \text{s.t.}& \quad \tr(\A\X\C_i\X^T)\leq \b_i,\quad \forall 1\leq i\leq m,\nonumber\\
  & \quad \A\succeq 0.
\label{eq:prob8frob}
\end{align}
Consider the following variational formulation of the problem:
\begin{align}
  \min_t \min_{\A}\ &\quad t+\gamma \|\A\|_F^2\nonumber\\
  \text{s.t.}& \quad \tr(\A\X\C_i\X^T)\leq \b_i,\quad \forall 1\leq i\leq m\nonumber\\
&\quad \tr(\X\C_0\X^T\A)\leq t\nonumber\\
  & \quad \A\succeq 0.
\label{eq:prob8frob1}
\end{align}
Note that for constant $t$, the inner minimization problem in the above problem is similar to \eqref{eq:probfrob} and hence can be kernelized. Corresponding optimization problem is given by:
\begin{align}
  \min_{\S, t} &\quad t+\gamma \tr(K_0 S K_0 S)\nonumber\\
  \text{s.t.}& \quad \tr( C_iK_0\S K_0)\leq b_i,\quad \forall 1\leq i\leq m\nonumber\\
&\quad \tr(\C_0K_0\S K_0)\leq t\nonumber\\
  & \quad S \succeq 0,
\label{eq:sdpfrobnew}
\end{align}
Similar to \eqref{eq:frobnew}, the above problem can be solved using convex optimization methods. 
\item {\bf Log determinant}: In this case we seek the solution to \eqref{eq:prob8} with minimum determinant. To this effect, we add a log-determinant regularization:
\begin{align}
  \min_{\A}\ &\quad \tr(\X\C_0\X^T\A)-\gamma \log\det \A\nonumber\\
  \text{s.t.}& \quad \tr(\A\X\C_i\X^T)\leq \b_i,\quad \forall 1\leq i\leq m,\nonumber\\
  & \quad \A\succeq 0.
\label{eq:prob8log}
\end{align}
The above regularization was also considered by \cite{brian_sdp}, which provided a fast projection algorithm for the case when each $C_i$ is a one-rank
matrix and discussed conditions for which the optimal solution to the
regularized problem is an optimal solution to the original SDP. 

Consider the following variational formulation of \eqref{eq:prob8log}:
\begin{align}
  \min_t\min_{\A}\ &\quad t-\gamma \log\det \A\nonumber\\
  \text{s.t.}& \quad \tr(\A\X\C_i\X^T)\leq \b_i,\quad \forall 1\leq i\leq m,\nonumber\\
& \quad \tr(\X\C_0\X^T\A)\leq t,\nonumber\\
  & \quad \A\succeq 0.
\label{eq:prob8log1}
\end{align}

 Note that the objective function of the inner optimization problem of \eqref{eq:prob8log1} satisfies the conditions of Theorem~\ref{thm:thm1}, and hence \eqref{eq:prob8log1} or equivalently \eqref{eq:prob8log} is kernelizable. 
\end{itemize}
\end{proof}
\section{Experimental Results}
In Section~\ref{sec:alg}, we presented metric learning as a constrained
LogDet optimization problem to learn a linear transformation, and we showed that the problem can be
efficiently kernelized.  Kernelization yields two fundamental advantages
over standard non-kernelized metric learning.  First, a non-linear kernel
can be used to learn non-linear decision boundaries common in applications
such as image analysis. Second, in Section~\ref{sec:itml_large}, we showed that the kernelized problem can be
learned with respect to a reduced basis of size $k$, admitting a learned
kernel parameterized by $O(k^2)$ values.  When the number of training
examples $n$ is large, this represents a substantial improvement over
optimizing over the entire $O(n^2)$ kernel matrix, both in terms of
computationally efficiency as well as statistical robustness.  

In this section, we present experiments from two domains: text analysis and imaging
processing.  As mentioned, image data sets tend to have highly non-linear
decision boundaries.  To this end, we learn a kernel matrix when the baseline kernel $K_0$ is the pyramid
match kernel, a method specifically designed for object/image recognition~\cite{iccv2005}.  In contrast, text data sets tend to perform quite
well with linear models, and the text experiments presented here have large
training sets.  We show that high quality metrics can be learned using a
relatively small set of basis vectors.

We evaluate performance of our learned distance metrics in the context of
classification accuracy for the $k$-nearest neighbor algorithm.  Our 
$k$-nearest neighbor classifier uses $k=10$ nearest neighbors (except for section~\ref{sec:image_exps} where we use $k=1$), breaking ties
arbitrarily.  We select the value of $k$ arbitrarily and expect to get slightly better accuracies using cross-validation. Accuracy is defined as the number of correctly classified
examples divided by the total number of classified examples.

For our proposed algorithms, pairwise constraints are inferred from true class 
labels.  For each class $i$, 100 pairs of points are randomly chosen from within
 class $i$ and are constrained to be similar, and 100 pairs of points are
drawn from classes other than $i$ to form dissimilarity constraints.  Given $c$
classes, this results in $100c$ similarity constraints, and $100c$
dissimilarity constraints, for a total of $200c$ constraints.  The upper and
lower bounds for the similarity and dissimilarity constraints are determined
empirically as the $1^{st}$ and $99^{th}$ percentiles of the distribution of
distances computed using a baseline Mahalanobis distance parameterized by
$\A_0$. Finally, the slack penalty parameter $\gamma$ used by our algorithms
is cross-validated using values $\{.01, .1, 1, 10, 100, 1000\}$.

All metrics are trained using data only in the training set.  Test instances
are drawn from the test set and are compared to examples in the training set
using the learned distance function.  The test and training sets are
established using a standard two-fold cross validation approach. For
experiments in which a baseline distance metric is evaluated (for example,
the squared Euclidean distance), nearest neighbor searches are again
computed from test instances to only those instances in the training set.

\subsection{Low-Dimensional Data Sets}
\begin{figure*}
\centering
\includegraphics[width=10cm]{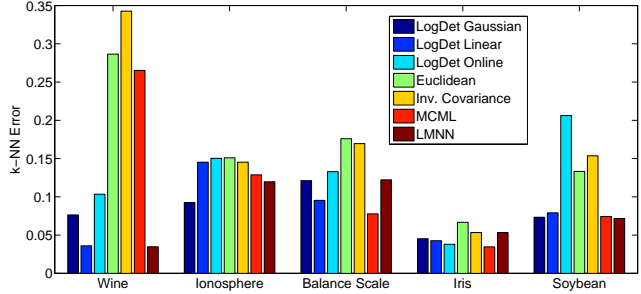}
\caption{Results over benchmark UCI data sets.  LogDet metric learning was
  run with in input space (LogDet Linear) as well as in kernel space with a
  Gaussian kernel (LogDet Gaussian).}
\label{fig:uci}
\end{figure*}
First we evaluate our metric learning method on the standard UCI datasets in
the low-dimensional (non-kernelized) setting, to directly compare with several
existing metric learning methods. 
In Figure~\ref{fig:uci}, we compare LogDet Linear ($K_0$ equals  the linear kernel) and the LogDet Gaussian ($K_0$ equals Gaussian kernel in kernel space) algorithms against existing
metric learning methods for $k$-NN classification. We use the squared
Euclidean distance, $d(\bm{x}, \bm{y}) =
(\bm{x}-\bm{y})^T(\bm{x}-\bm{y})$ as a baseline method.  We also use a
Mahalanobis distance parameterized by the inverse of the sample
covariance matrix.  This method is equivalent to first performing a
standard PCA whitening transform over the feature space and then
computing distances using the squared Euclidean distance.  We compare
our method to two recently proposed algorithms: Maximally Collapsing
Metric Learning~\cite{roweis} (MCML), and metric learning via Large
Margin Nearest Neighbor~\cite{wein} (LMNN).  Consistent with existing
work such as \cite{roweis}, we found the method of~\cite{xing} to be very slow
and inaccurate, so the latter was not included in our experiments.  As seen in Figure~\ref{fig:uci}, LogDet Linear and LogDet Gaussian algorithms obtain somewhat higher accuracy for most of the datasets. 

\begin{figure*}[ht]
\hspace*{-.3in}
$\begin{array}{c@{\hspace{.2in}}c@{\hspace{.2in}}c}
  \includegraphics[width=3.1in]{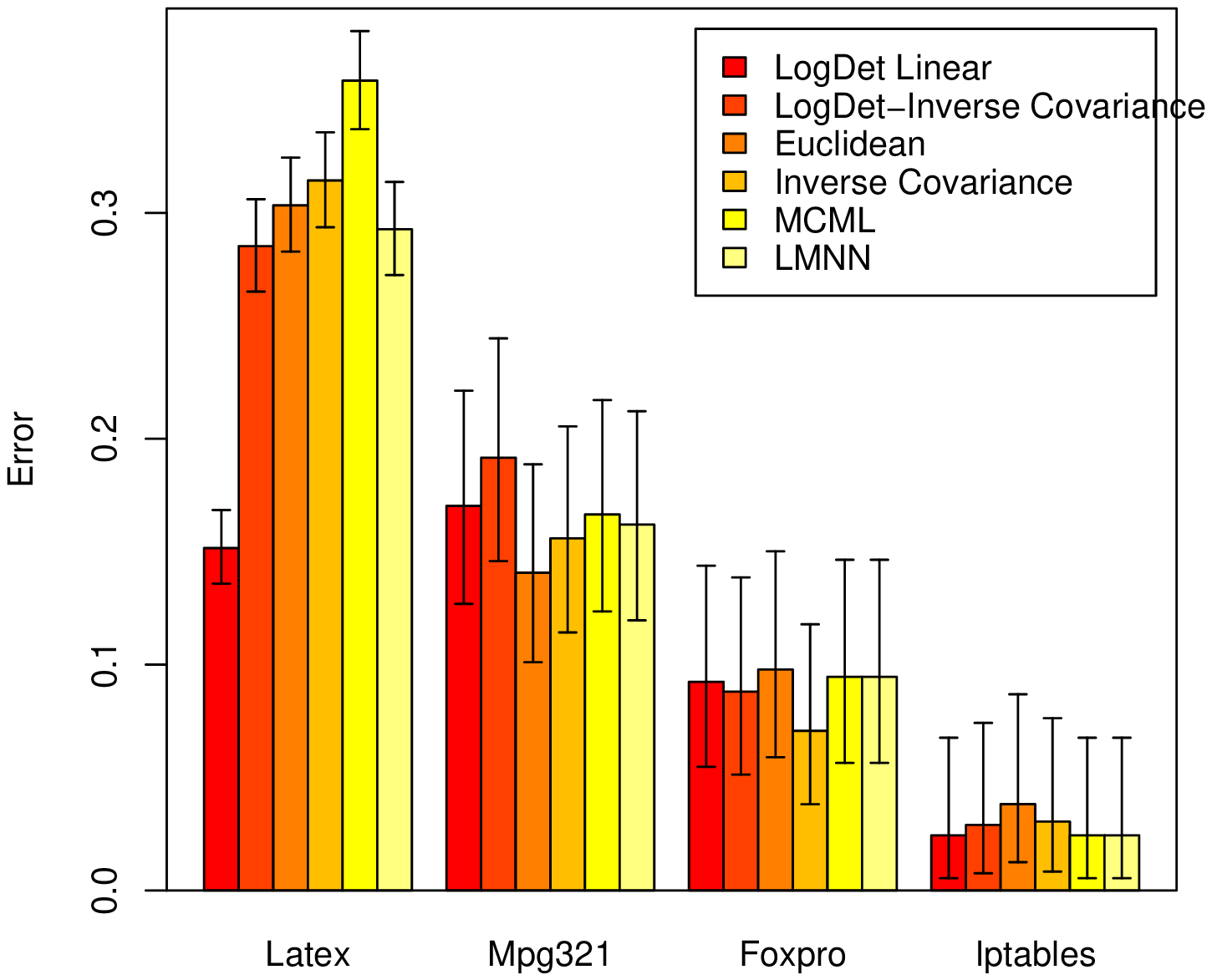} &
  \includegraphics[width=3.1in]{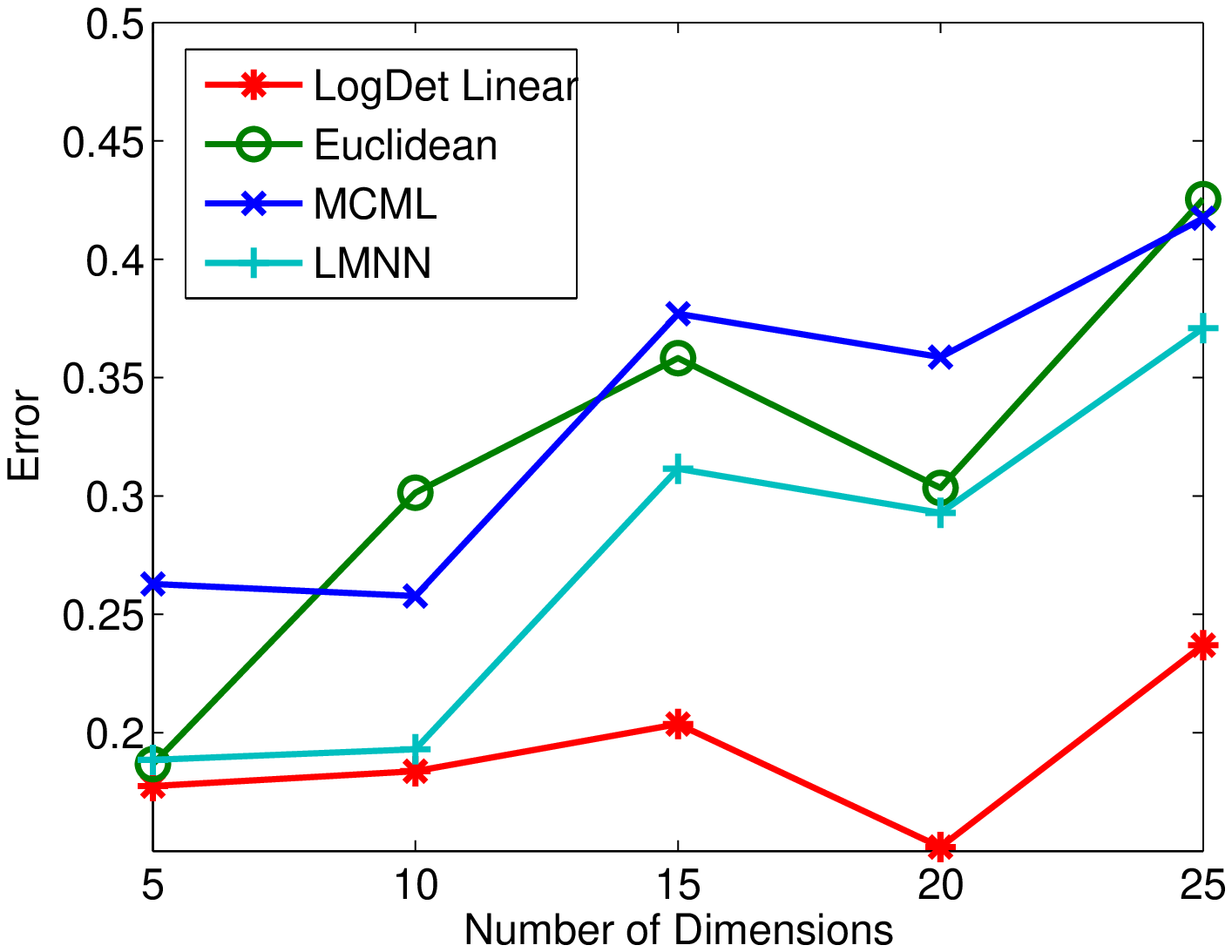} \\[-12pt]
  (a) \mbox{ Clarify Datasets}  & (b) \mbox{ Latex}
\end{array}$
\caption{Classification error rates for $k$-nearest neighbor software support  via different learned metrics.  We see in figure (a)  that
  LogDet Linear is the only algorithm to be optimal (within the $95\%$
  confidence intervals) across all datasets.  LogDet is also robust at
  learning metrics over higher dimensions.  In (b), we see that the error
  rate for the Latex dataset stays relatively constant for LogDet Linear.
}
\label{fig:results}
\end{figure*}

In addition to our evaluations on standard UCI datasets, we also apply our algorithm to the recently
proposed problem of nearest neighbor software support for the Clarify system~\cite{clarify}.  The basis
of the Clarify system lies in the fact that modern software design promotes
modularity and abstraction.  When a program terminates abnormally, it is often unclear which component should
be responsible for (or is capable of) providing an error report.  The system works by monitoring a set of
predefined program features (the datasets presented use function
counts) during program runtime which are then used by a classifier in the
event of abnormal program termination.  Nearest neighbor searches are
particularly relevant to this problem.  Ideally, the neighbors returned
should not only have the correct class label, but should also represent
those with similar program configurations or program inputs.  Such a
matching can be a powerful tool to help users diagnose the root cause of
their problem.  The four datasets we use correspond to the following softwares: Latex (the document
compiler, 9 classes), Mpg321 (an mp3 player, 4 classes), Foxpro (a database
manager, 4 classes), and Iptables (a Linux kernel application, 5 classes).

Our experiments on the Clarify system, like the UCI data, are over
fairly low-dimensional data.
It was shown~\cite{clarify} that high classification accuracy can be obtained
by using a relatively small subset of available features.  Thus, for each
dataset, we use a standard information gain feature selection test to obtain
a reduced feature set of size 20.  From this, we learn metrics for $k$-NN
classification using the methods developed in this paper. Results are given in Figure~\ref{fig:results}(b). The LogDet Linear algorithm yields
significant gains for the Latex benchmark.  Note that for datasets where
Euclidean distance performs better than using the inverse covariance metric,
the LogDet Linear algorithm that normalizes to the standard Euclidean
distance yields higher accuracy than that regularized to the inverse
covariance matrix (LogDet-Inverse Covariance). In general, for the Mpg321, Foxpro,
and Iptables datasets, learned metrics yield only marginal gains over the
baseline Euclidean distance measure.

Figure~\ref{fig:results}(c) shows the error rate for the Latex datasets with
a varying number of features (the feature sets are again chosen using the
information gain criteria).  We see here that LogDet Linear is surprisingly robust. 
Euclidean distance, MCML, and LMNN all achieve their best error rates for
five dimensions.  LogDet Linear, however, attains its lowest error rate of .15 at
$d=20$ dimensions.

In Table~\ref{tbl:time}, we see that LogDet Linear generally learns metrics
significantly faster than other metric learning algorithms.  The
implementations for MCML and LMNN were obtained from their respective
authors.  The timing tests were run on a dual processor 3.2 GHz Intel Xeon
processor running Ubuntu Linux.  Time given is in seconds and represents the
average over 5 runs.

\begin{table}
\centering
\caption{Training time (in seconds) for the results presented in Figure~\ref{fig:results}(b).  }
\label{tbl:time}
\begin{tabular}{|l||c|c|c|}
\hline
Dataset & LogDet Linear & MCML & LMNN \\
\hline
\hline
Latex & {\bf 0.0517} & 19.8 & 0.538 \\ \hline
Mpg321 & {\bf 0.0808} & 0.460 & 0.253 \\ \hline
Foxpro & {\bf 0.0793} & 0.152 & 0.189 \\ \hline
Iptables & 0.149 & {\bf 0.0838} & 4.19 \\ \hline
\end{tabular}
\end{table}

We also present some semi-supervised clustering results for two of the UCI data sets.  Note that both
MCML and LMNN are not amenable to optimization subject to pairwise distance
constraints.  Instead, we compare our method to the semi-supervised
clustering algorithm HMRF-KMeans~\cite{sugato}.  We use a standard 2-fold
cross validation approach for evaluating semi-supervised clustering results.
Distances are constrained to be either similar or dissimilar, based on class
values, and are drawn only from the training set.  The entire dataset is
then clustered into $c$ clusters using $k$-means (where $c$ is the number of classes) and
error is computed using only the test set.  Table~\ref{tbl:ss} provides
results for the baseline $k$-means error, as well as semi-supervised
clustering results with 50 constraints.

\begin{table}
\centering
\caption{Unsupervised $k$-means clustering error using the baseline squared Euclidean distance, along with semi-supervised
  clustering error with 50 constraints.}
\label{tbl:ss}
\begin{tabular}{|l||c|c|c|}
\hline
Dataset & Unsupervised & LogDet Linear & HMRF-KMeans \\
\hline 
\hline 
Ionosphere  &0.314  & {\bf 0.113}  & 0.256 \\ \hline 
Digits-389  &0.226  & {\bf 0.175 }& 0.286 \\\hline 
\end{tabular}
\end{table}
\subsection{Metric Learning for Object Recognition}
\label{sec:image_exps}
\begin{figure}
\centering
\includegraphics[width=\textwidth]{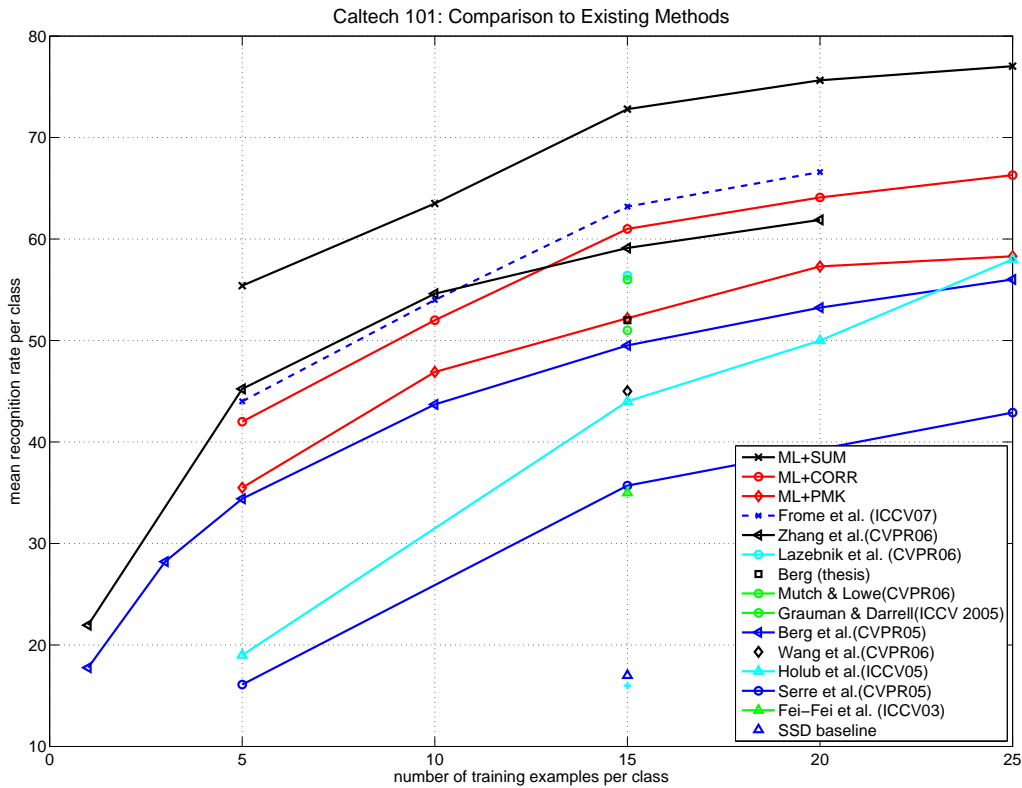}
\caption{Caltech-101: Comparison of LogDet based metric learning method with
  other state-of-the-art object recognition methods. Our method outperforms all other single metric/kernel approaches.  ML+SUM refers to our learned kernel when the average of four kernels (PMK~\cite{iccv2005}, SPMK~\cite{spmk}, Geoblur-1, Geoblur-2~\cite{geoblur}) is the base kernel, ML+PMK refers to the learned kernel over the pyramid match~\cite{iccv2005} as the base kernel, and ML+CORR refers to the learned kernel when the correspondence kernel of~\cite{zhang-cvpr06} is the base kernel.}\label{fig:caltech101all}
\end{figure}

\begin{figure}[t]
\centering
 \includegraphics[width=10cm]{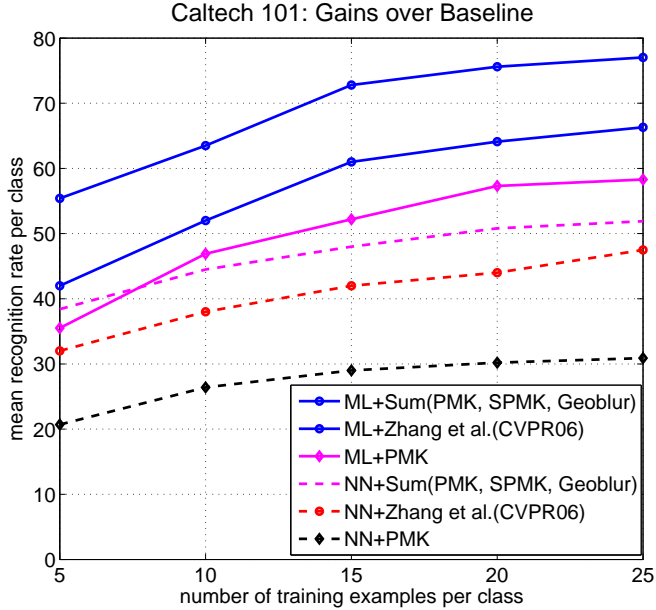}
\caption{Object recognition on the Caltech-101 dataset. Our learned kernels significantly improve NN recognition accuracy relative to their non-learned counterparts, the SUM (average of four kernels), the CORR and PMK kernels.} \label{fig:caltech101}
\end{figure}
Next we evaluate our method over high-dimensional data applied to the object-recognition task using Caltech-101 \cite{caltech101}, a common benchmark for this task. The goal is to predict the category of the object in the given image using a $k$-NN classifier. 

We compute distances between images using learning kernels with three
different base image kernels: 1) PMK: Grauman and Darrell's pyramid match kernel
~\cite{iccv2005} applied to SIFT features, 2) CORR: the kernel designed by
~\cite{zhang-cvpr06} applied to geometric blur features
, and 3) SUM: the average of four image kernels, namely, PMK \cite{iccv2005}, Spatial PMK \cite{spmk}, Geoblur-1, and Geoblur-2 \cite{geoblur}. Note that the underlying
dimensionality of these embeddings are typically in the millions of dimensions.

We evaluate the effectiveness of metric/kernel learning on this dataset.  We pose a $k$-NN classification task, and evaluate both the
original (SUM, PMK or CORR) and learned kernels. We set $k=1$ for our experiments; this value was chosen arbitrarily. We vary the number of training examples $T$ per class for the database, using the remainder as test examples, and measure accuracy in terms of the mean recognition rate per class, as is standard practice for this dataset.

Figure~\ref{fig:caltech101all} shows our results relative to several other existing techniques that have been applied to this dataset.  Our approach outperforms all existing single-kernel classifier methods when using the learned CORR kernel: we achieve 61.0\% accuracy
for $T=15$ and 69.6\% accuracy for $T=30$.  Our learned PMK achieves 52.2\%
accuracy for $T=15$ and 62.1\% accuracy for $T=30$. Similarly, our learned
SUM kernel achieves $73.7\%$ accuracy for $T=15$. Figure~\ref{fig:caltech101}
specifically shows the comparison of the original baseline kernels for NN classification.  The plot reveals gains in 1-NN classification accuracy; notably, our learned kernels with simple NN classification also outperform the baseline kernels when used with SVMs~\cite{zhang-cvpr06,iccv2005}.   

\subsection{Metric Learning for Text Classification}
Next we present results in the text domain.  Our text datasets are created
by standard bag-of-words Tf-Idf 
representations.  Words are stemmed using a standard Porter stemmer and
common stop words are removed, and the text models are limited to the 5,000
words with the largest document frequency counts.  We provide experiments
for two data sets: CMU Newsgroups \cite{ng20}, and Classic3 \cite{classic3}.  Classic3 is a relatively small 3 class
problem with 3,891 instances.  The newsgroup data set is much larger, having
20 different classes from various newsgroup categories and 20,000 instances.

As mentioned earlier, our text experiments use a linear kernel, and we use a set of
basis vectors that is constructed from the class labels via the following
procedure.   Let $c$ be the number of distinct classes and let $k$ be the
size of the desired basis.  If $k = c$, then each class mean $r_i$ is computed
to form the basis $R = [ {\bm r_1} \ldots {\bm r_c} ]$. If $k<c$ a similar
process is used but restricted to a randomly selected subset of $k$ classes.
If $k > c$, instances within each class are clustered into  approximately
$\frac{k}{c}$ clusters.  Each cluster's mean vector is then computed to form
the set of low-rank basis vectors $R$.

\begin{figure*}[t]
\centering
\begin{tabular}{c}
$\begin{array}{c@{\hspace{.2in}}c}
\includegraphics[width=2.5in]{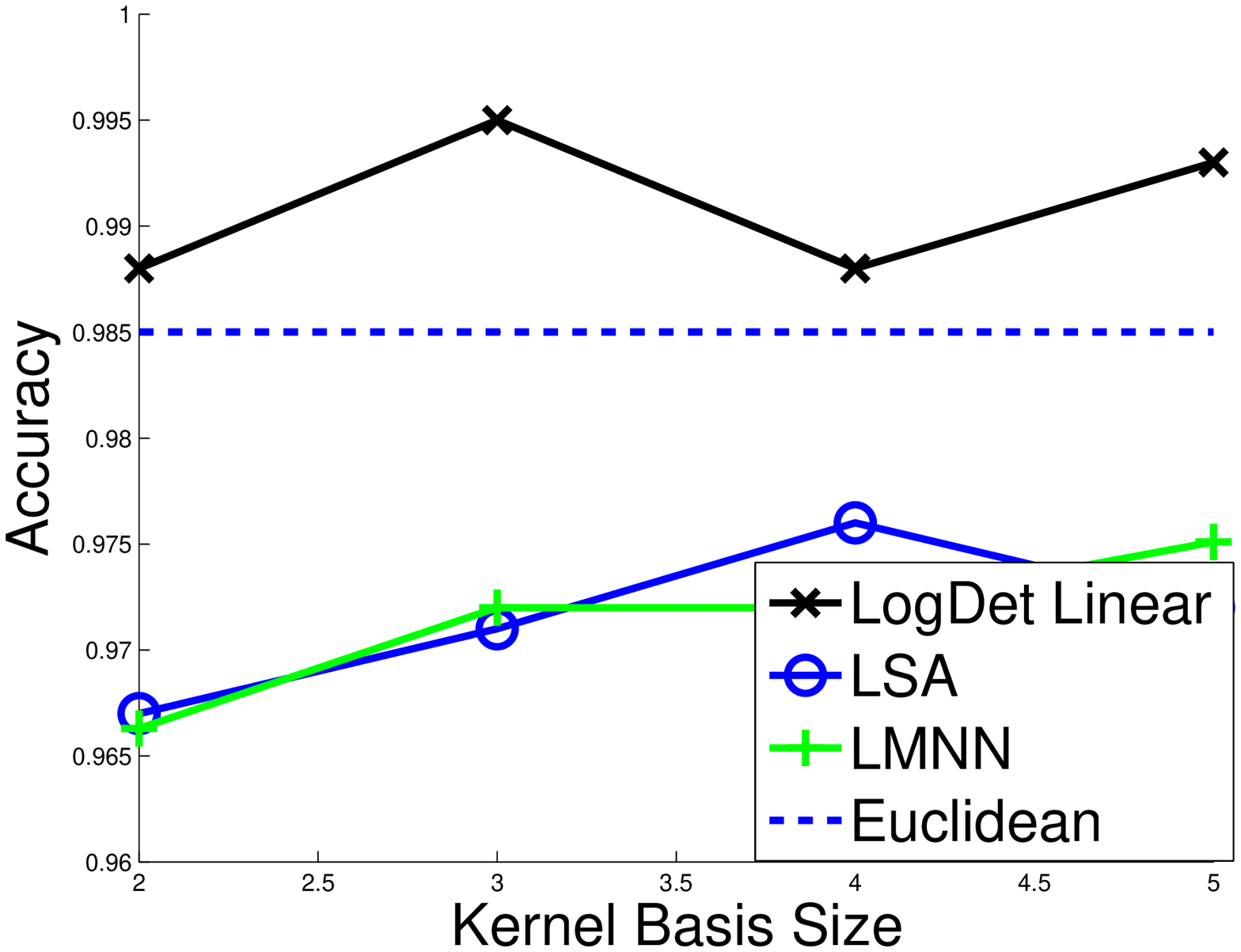} &
\includegraphics[width=2.5in]{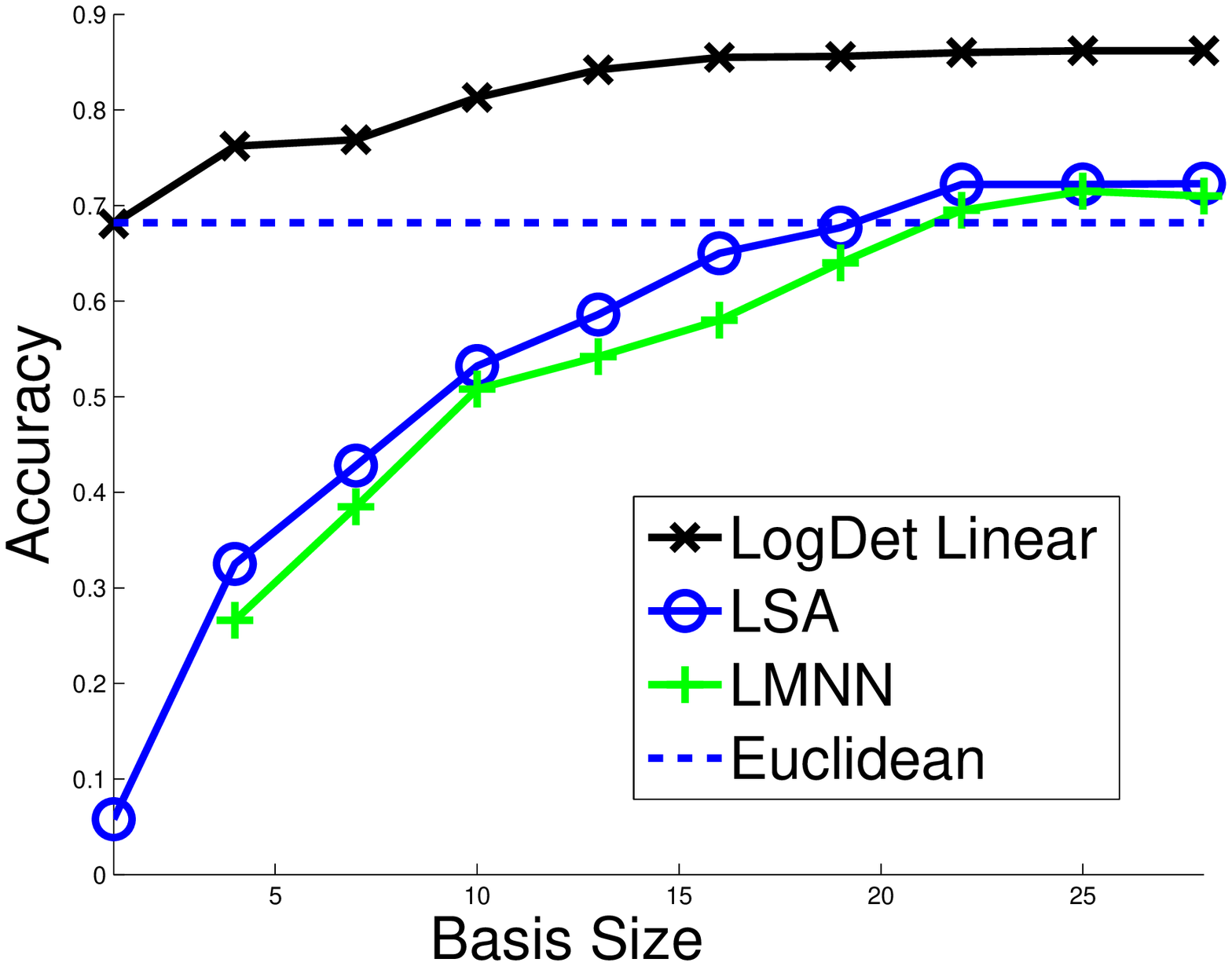} \\
(a) \mbox{ Classic3} & (b) \mbox{ 20-Newsgroups } 
\end{array}
$\\
\end{tabular}
\caption{Classification accuracy for our Mahalanobis metrics learned over basis of different dimensionality. Overall, our method (LogDet Linear) significantly outperforms existing methods.}
\label{fig:text_results}
\end{figure*}

Figure~\ref{fig:text_results} shows classification accuracy across bases of
varying sizes for the Classic3 dataset, along with the newsgroup data
set. As baseline measures, the standard squared Euclidean distance is shown,
along with Latent Semantic Analysis (LSA)~\cite{deerwester90lsa}, which works
by projecting the data via principal components analysis (PCA), and
computing distances in this projected space.  Comparing our algorithm to the
baseline Euclidean measure, we can see that for smaller bases, the accuracy
of our algorithm is similar to the Euclidean measure.  As the size of the basis increases, our method obtains significantly higher accuracy compared to the baseline Euclidean measure.


\section{Conclusions}
In this paper, we have considered the general problem of learning a linear
transformation of input data and applied it to the problem of learning a
metric over 
high-dimensional data or feature space implicitly. 
$\phi(\bm{x}_i)^T A \phi(\bm{x}_j)$.  
We first showed that the LogDet divergence is a useful loss for learning a linear transformation (or performing metric learning) in kernel space,
as the algorithm can easily be generalized to work in kernel space.  We then
proposed an algorithm based on Bregman projections to learn a kernel function over the data-points efficiently. We also show that our learned metric can be restricted to a small dimensional basis efficiently, hence scaling our method to large datasets with high-dimensional feature space.  Then we considered a
larger class of convex loss functions for learning the metric/kernel using a linear transformation of the data; we saw that many loss functions can lead to kernelization, though the resulting
optimizations may be more expensive to solve than the simpler LogDet
formulation.  Finally, we presented some experiments on benchmark data,
high-dimensional vision, and text classification problems, demonstrating our
method compared to several existing state-of-the-art techniques.

There are several directions of future work.  To facilitate even larger data
sets than the ones considered in this paper, online learning methods are one
promising research direction; in~\cite{onlinemetric_nips}, an online learning
algorithm was proposed based on LogDet regularization, and this remains a
part of our ongoing efforts.  Recently, there has been some interest in
learning multiple local metrics over the data;~\cite{weinnew} considered this
problem.  We plan to explore this setting with the LogDet divergence, with a
focus on scalability to very large data sets.
\section*{Acknowledgements}
This research was supported by NSF grant CCF-0728879. We would also like to acknowledge Suvrit Sra for various helpful discussions. 
{\footnotesize
\bibliographystyle{biburl}
\bibliography{refs}
}
\end{document}